%% file: Main_neurips_2022.tex
\documentclass{article}

\usepackage[final]{neurips_2022}

\usepackage[utf8]{inputenc} % allow utf-8 input
\usepackage[T1]{fontenc}    % use 8-bit T1 fonts
\usepackage{hyperref}       % hyperlinks
\usepackage{url}            % simple URL typesetting
\usepackage{booktabs}       % professional-quality tables
\usepackage{amsfonts}       % blackboard math symbols
\usepackage{nicefrac}       % compact symbols for 1/2, etc.
\usepackage{microtype}      % microtypography
\usepackage{xcolor}         % colors

\usepackage{amsmath}
\usepackage{amssymb}
\usepackage{mathtools}
\usepackage{amsthm}
\usepackage{caption}
\usepackage{subcaption}
\usepackage{wrapfig}
\usepackage{floatrow}

\usepackage[textsize=tiny]{todonotes}

\usepackage{array}
\newcolumntype{P}[1]{>{\centering\arraybackslash}p{#1}}
\newcolumntype{M}[1]{>{\centering\arraybackslash}m{#1}}

\usepackage{pgf, tikz}
\usetikzlibrary{arrows, automata}
\usetikzlibrary{bayesnet}

\input{preamble_commands.tex}

\input{new_commands.tex}

\theoremstyle{plain}
\newtheorem{theorem}{Theorem}[section]

\newtheorem{lemma}[theorem]{Lemma}

\theoremstyle{definition}
\newtheorem{definition}[theorem]{Definition}

\theoremstyle{remark}

\title{
The Missing Invariance Principle Found --  \\ 
the Reciprocal Twin of 
Invariant Risk Minimization 

}

\author{%
  Dongsung Huh\thanks{Equal contribution} \\
  MIT-IBM Watson AI Lab\\
  Cambridge, MA 02142 \\
  \texttt{huh@ibm.com} \\
   \And
   Avinash Baidya\footnotemark[1] \\
   Department of Physics and Astronomy \\
   University of California \\
   Davis, CA 95616 \\
   \texttt{aavinash@ucdavis.edu} \\
}

\begin{document}

\maketitle

\input{abstract.tex}

\input{new_intro.tex}

\input{theory}

\input{linear_setting}

\input{non_linear_setting}

\input{discussion}

\begin{ack}
This research was done as part of Avinash Baidya's internship at IBM. 
We thank Joel Dapello for helpful discussions.
% and comments.
We also thank the anonymous reviewers for constructive comments. 
\end{ack}

\bibliography{referenceIRM}
\bibliographystyle{apalike}

\newpage

\input{checklist}

\appendix
\onecolumn

\input{appendix}

\end{document}

%% file: preamble_commands.tex
\usepackage{amsfonts}       % blackboard math symbols
\usepackage{mathtools}

\usepackage{multirow}

\usepackage{color}
\usepackage{bm}

\definecolor{white}{rgb}{1,1,1}
\definecolor{blue}{rgb}{0.1,0.1,1}
\definecolor{cyan}{rgb}{0.55,0.6,0.8}  %{0.0,0.7,1}
\definecolor{red}{rgb}{0.8,0.2,0.2}  %{1,0.1,0.1}
\definecolor{green}{rgb}{0,0.8,0} %{0.2,0.6,0.2} 
\definecolor{purple}{rgb}{0.6,0,1}
\definecolor{orange}{rgb}{1,0.6,0.3}

\newcommand{\grey}{\color{cyan}}
\newcommand{\blue}{\color{black}}
\newcommand{\cyan}{\color{black}}
\newcommand{\red}{\color{black}} 
\newcommand{\green}{\color{black}}

%% file: new_commands.tex
\DeclareMathOperator*{\argmin}{arg\,min}

\newcommand{\trans}{^\intercal}
\newcommand{\loss}{\mathcal{L}}
\newcommand{\expect}{\mathbb{E}}
\newcommand{\out}{o}    \newcommand{\Out}{O} %{\mathbb{R}} %
\newcommand{\labl}{y} \newcommand{\Labl}{Y}
\newcommand{\inp}{x}    \newcommand{\Inp}{X}
 
\newcommand{\latent}{z} \newcommand{\Latent}{Z}
\newcommand{\causal}{i} %{o}
\newcommand{\spu}{s} 
\newcommand{\Env}{E}

\newcommand{\represent}{f} %{\phi} 
\newcommand{\modifier}{\psi} 
\newcommand{\predictor}{f}

\newcommand{\IRMorig}{{\text{IRM-$\text{v1}$}}}

\newcommand{\MRIl}{{\text{MRI-$\text{v1}$}}}

%% file: abstract.tex
\begin{abstract}
    Machine learning models often generalize poorly to out-of-distribution (OOD) data as a result of relying on features that are spuriously correlated with the label during training. Recently, the technique of Invariant Risk Minimization (IRM) was proposed to learn predictors that only use invariant features by conserving the feature-conditioned label expectation $\expect_e[\labl|\predictor({\inp})]$ across environments.
    However, more recent studies have demonstrated that IRM-v1, a practical version of IRM, can fail {\cyan in various settings}. 
    Here, we identify a fundamental {\red design} flaw of IRM formulation that causes the failure.
    We then introduce a complementary notion of invariance, MRI,  based on conserving the label-conditioned feature expectation $\expect_e[\predictor({\inp})|\labl]$,
    % across environments, 
    which is free of this flaw. 
    Further, we introduce a simplified, practical version of the MRI formulation called  $\MRIl$. 
    % We note that this constraint is convex which confers it with an advantage over $\IRMorig$, {\cyan which imposes non-convex constraints.}
    We prove that for general linear problems,  $\MRIl$ guarantees invariant predictors given sufficient number of environments. 
    We also empirically demonstrate that 
    MRI-v1 strongly out-performs IRM-v1 and consistently achieves near-optimal OOD generalization in  image-based nonlinear problems. 
\end{abstract}

%% file: new_intro.tex
\section{Introduction}

% {\red

% Outline:
% \begin{itemize}
%     \item Typical machine learning algorithms learn spurious feature in the dataset. Example - cow vs camel.
%     \item However, there has been recent interest in this problem. The IRM model stands out as it one of the first method to demonstrate success on high-dimensional data. However, many have shown that it doesn't work.
%     \item In this work, we argue that there is a major flaw in the IRM formulation that prevents it from learning invariance. To correct for this, we present a complementary notion of invariance which is based on preserving class-conditional feature expectation that 
    
% \end{itemize}}

Deep learning models have shown tremendous success over the past decade. 
These models show great generalization properties when tested on the same distribution as the training dataset (in-distribution generalization).
% However, typical machine learning models are trained and tested on datasets belonging to the same distribution (in-distribution). 
However, these models often show catastrophic failure when tested on out-of-distribution dataset, 
% These models when tested on out-of-distribution data fail catastrophically as a result of relying on 
revealing that they {\green learned features} that are spuriously correlated to the label in the given training domains but do not generalize to the testing domains. For example, deep networks trained on pictures of cow with only grassy backgrounds in the training domain will use the background color as the predictive feature which is easier to learn and generalize poorly to pictures of cow with a dessert background.

Recently, % Fortunately, 
there has been a growing interest in developing models that generalize well across multiple domains. In particular, there has been a recent body of works that focus on developing algorithms that attempt to learn invariant predictors \citep{arjovsky2019invariant, ahuja2021invariance, peters2016causal, rojas2018invariant, heinze2018invariant}. Invariant Risk Minimization (IRM) and its practical version, $\IRMorig$, {\green have garnered significant attention} 
as one of the initial methods that are compatible with deep learning techniques. However, several follow-up studies have empirically demonstrated that  IRM-v1 is unreliable at learning invariant representations \citep{kamath2021does, rosenfeld2020risks, gulrajani2020search, ahuja2020empirical, ahuja2021invariance}. 
% While many of these studies have analyzed various limitations of IRM-v1,
Here, we identify a fundamental flaw of IRM formulation that causes this limitation and propose a new method that is free of this flaw.

\input{related_works}

\paragraph{Our Contributions}

% In this work, 
% we identify a fundamental flaw in the IRM formulation 
% that occurs as result of the extraneous constraint it adds, on top of conserving $\expect_e[\labl|\predictor(\inp)]$ across environments. 
We introduce a variational formulation of IRM,
%  {\red that is more generalizable}, 
and show that it can be modified to yield a new complementary notion of invariance, called MRI. 
We show that IRM has a fundamental flaw due to the indirect way of imposing invariance which leads to the failure of IRM-v1. In constrast, MRI is shown to be free of this flaw. 
% We then propose the MRI formulation that is free of this flaw. 
% This method uses a new notion of invariance that is based on conserving the class-conditional feature expectation across domains. 
% Furthermore, we introduce a simplified, practical version of this formulation called $\MRIl$ that allows us to scale this new notion of invariance to high-dimensional data. We note that this simplified version provides an advantage over $\IRMorig$ as it has a convex constraint, as compared to the non-convex constraint of $\IRMorig$. 
% 
We prove that  $\MRIl$ can guarantee invariant predictors in general linear problem settings given sufficient environments. 
We also show empirical demonstrations that MRI strongly out-performs IRM and consistently achieves near-optimal OOD generalization in nonlinear image-based problems.%
\footnote{Code available at
\url{https://github.com/IBM/MRI}.
}

%% file: related_works.tex
\paragraph{Related Works}
% {\purple %[copied from Risks of IRM]
% Works on learning deep invariant representations vary considerably: some search for a domain-invariant representation (Muandet et al., 2013; Ganin et al., 2016), i.e. invariance of the distribution $P(\phi(x))$, typically used for domain adaptation (Ben-David et al., 2010; Ganin \& Lempitsky, 2015; Zhang et al., 2015; Long et al., 2018), with assumed access to labeled or unlabeled data from the target distribution. Other works instead hope to find representations that are conditionally domain-invariant, with invariance of $P(\phi(x) | y)$ (Gong et al., 2016; Li et al., 2018). However, there is evidence that invariance may not be sufficient for domain adaptation (Zhao et al., 2019; Johansson et al., 2019). In contrast, this paper focuses instead on domain generalization (Blanchard et al., 2011; Rosenfeld et al., 2021), where access to the test distribution is not assumed.
% Recent works on domain generalization, including the objectives discussed in this paper, suggest invariance of the feature-conditioned label distribution. In particular, Arjovsky et al. (2019) only assume invariance of $\expect[y | \phi(x)]$; follow-up works rely on a stronger assumption of invariance of higher conditional moments (Krueger et al., 2020; Xie et al., 2020; Jin et al., 2020; Mahajan et al., 2020; Bellot \& van der Schaar, 2020). Though this approach has become popular in the last year, it is somewhat similar to the existing concept of covariate shift (Shimodaira, 2000; Bickel et al., 2009)
% }

There has been considerable work in the field of learning invariant representations. They vary from learning domain-invariant feature representations conserving $P(f(x))$ using kernel methods \citep{muandet2013domain, ghifary2016scatter, hu2020domain}, variational autoencoder \citep{ilse2020diva}, and adversarial networks \cite{ganin2016domain, long2018conditional, akuzawa2019adversarial, albuquerque2019generalizing} to learning invariant class-conditional features $P(f(x)|y)$ \citep{gong2016domain, li2018domain} in the context of domain adaptation, which assumes access to the test distribution for adaptation. % However, in these cases algorithms have access to the test distribution. 
There is also a large body of work to learn invariant representations in the field of domain generalization that doesn't assume access to test distributions. This includes imposing invariance of $\expect_e[\labl|\predictor(x)]$ \citep{arjovsky2019invariant} with information bottleneck constraint (\cite{ahuja2021invariance}), imposing object-invariant condition (\cite{mahajan2021domain}), using domain inference \citep{creager2021environment}, model calibration \citep{wald2021calibration}, and others \citep{krueger2021out, li2018learning, shankar2018generalizing}.

%% file: theory.tex
\section{Problem Formulation}

% \paragraph{Environments} %{Data generation process}

% {\cyan We assume the data generating process for $\mathcal{E}, \Inp, \Labl$ follows the causal graph in Fig~\ref{fig:sem}  as considered in \citep{arjovsky2019invariant,rosenfeld2020risks}}.
% 
% We consider the graphical model of data generation process introduced in \citep{rosenfeld2020risks} See Fig~\ref{fig:sem}. 

Consider a set of environments $\mathcal{E} = \{e\}$, 
% $\mathcal{E} = \{e_i\}_{i=1}^{E}$,
each of which defines a distribution $P_e(\inp,\labl)$ over inputs and labels,
from which the dataset of the environment is drawn $\mathcal{D}_e \equiv \{(x^{e}_j \in \mathbb{R}^{d}, y^{e}_j \in \mathbb{R})\}$.
% and a corresponding set of dataset $\{\mathcal{D}_e \}_{e\in\mathcal{E}}$, 
% 
% the data generating process for $\mathcal{E}, \Inp, \Labl$
The distribution $P_e(\inp,\labl)$ is assumed to be generated according to the causal graph in Fig~\ref{fig:sem}
% {\red as introduced in} 
\citep{rosenfeld2020risks}, 
which includes latent features that are invariantly  $\latent_\causal$
or spuriously $\latent_\spu$ correlated with the {\green label},
from which the observation $\inp$ is generated.%
% {\red Their causal parent is $\labl$. }
% 
\footnote{\green 
While Fig~\ref{fig:sem} only shows %considers 
the causal direction $\Labl \to \Latent_\causal$,
\red 
the other direction 
% in which the invariant feature $\latent_\causal$ causes the label $\labl$ 
$\Latent_\causal \to \Labl$ is 
also consistent with %applicable to 
our analysis here %{\red our results} on IRM and MRI, %are 
% also consistent with %applicable to 
% 
\red %MRI requires 
as long as $P(\labl)$ is independent from the environment index $e$.
% Note that \cite{arjovsky2019invariant} considered $\Latent_\causal \to \Labl$
% % which was considered in \cite{arjovsky2019invariant} %(Fig~3) 
% with the difference being that they allow both the invariant feature and label distribution to vary across environments. 
% Also, 
% while \cite{arjovsky2019invariant} focused on the causal graph with $\Latent_\causal \to \Labl$, the 
}

\input{tikz_figure}

% \paragraph{Environmental Loss}
% Given an environment $e \in \mathcal{E}$ with data distribution $P_e(\inp,\labl)$, a predictor $\predictor: \Inp \to \Out$ 
% the {\it environmental loss}  of the predictor is %defined as
The risk of a predictor $\predictor: \Inp \to \Out$ in environment $e$ 
is defined as the population average 
\begin{align}
	\label{eq:env_loss}
    \loss_e(f) &\equiv \expect_{P_e(\inp,\labl)}  [l(f(\inp), \labl)]  \\
	\label{eq:env_loss_IRM}
    &= \expect_{P_e(\out)} [  \expect_{P_e( \labl | \out)}  [ l(\out, \labl) ]] \\
	\label{eq:env_loss_MRI}
    &= \expect_{P(\labl)} [  \expect_{P_e(\out | \labl)}  [ l(\out, \labl) ]] 
\end{align}
where $\out = \predictor(\inp)$ is the predictor's output.
Here, we consider standard convex loss functions 
$l: \Out \times \Labl \to \mathbb{R}_{\geq 0}$,
% {\blue that are convex with respect to  $\out$ as well as $\labl$,}
% Equation~(\ref{eq:env_loss_IRM}\ref{eq:env_loss_MRI}) 
% 
including the square loss 
$l_\text{sq} (\out, \labl) =  \frac {1}{2} (\out - \labl )^2$
for regression  %$\hat{Y} = \mathbb{R}$ 
($\Out,\Labl \subseteq \mathbb{R}$)
and the binary-cross-entropy (BCE) loss
% $l_\text{log} (\out, \labl)  =  - (1+\labl) \log((1+\tanh(\out/2))/2) - (1-\labl) \log ((1-\tanh(\out/2))/2)$%
% for binary classification  ($\Out \subseteq \mathbb{R} $, $\Labl  \subseteq [-1,1]$). % {\cyan ($\Labl = \{-1,1\}$)}%
% 
$l_\text{log} (\out, \labl) = - (1+\labl) \log(\eta(\out)) - (1-\labl) \log (1-\eta(\out))$ for binary classification
% \footnote{
% 	This is equivalent to the better known form 
%  	$l_\text{log} (\out, \labl) =  -\labl \log(\eta(\out)) - (1-\labl) \log (1-\eta(\out)) $  for $\Labl \subseteq [0,1]$. %where $\eta(\out)=1/(1+e^{-\out})$. } 
($\Out \subseteq \mathbb{R} $, $\Labl  \subseteq [-1,1]$), where $\eta(\out) \equiv 1/(1+e^{-\out})$ is the sigmoid function.
%
% \footnote{\cyan
% Although classification datasets often involve discrete labels in $\{-1,1\}$,  classification {\green problems} can be defined %are well-defined 
% to use soft-labels in the simplex $[-1,1]$, {\it e.g.} in knowledge distillation. Generally, a $N$-way classification problem	can be defined %formulated 
% with a $N-1$ dimensional simplex coding scheme \citep{mroueh2012multiclass}.}

% {\blue Note that the loss functions are well-defined on the real-valued label $\labl$, even though only discrete values of $\labl$ are used for binary classification}

% {\purple
% [Note that this is different from 
% what was considered in IRM.
% IRM did not have SEM..	]
% }
% \subsection{{\cyan Variational-principle} Algorithms}
\section{IRM vs MRI Invariance}  \label{sec:IRMvsMRI}
\subsection{IRM Paradigm}
\subsubsection{Original formulation}

% \subsection{Necessary condition for Invariant Optimal Predictor}
% \subsection{Invariant Risk Minimization: Background -- Notions of Invariance}

% \subsubsection{Notion of Invariance}

% IRM considers a predictor $\predictor: \Inp \to \Out$ and a {\cyan modifier} $\modifier: \Out \to \Out$%

% The IRM paradigm is defined as below \citep{arjovsky2019invariant}:
% {\green proposes to impose} this invariance via the following formulation:
%  is {\cyan formally} defined as follows:
% 
\begin{definition}{\citep{arjovsky2019invariant}}%[IRM Invariance]
\label{def:IRM}
The feature representation $\represent: \Inp \to \Latent$ elicits an {\it {\red IRM} invariant} predictor $\modifier \circ \represent: \Inp \to \Out$ over a set of environments $\mathcal{E}$, 
if there exists $\modifier: \Latent \to \Out$ such that {\blue $\modifier$ is simultaneously optimal for all environments in $\mathcal{E}$},
    {\it i.e.}
    \begin{align}
        \forall e \in \mathcal{E}, ~~~~~~~~~~~~~~~~~~ 
        \modifier \in \argmin_{\bar{\modifier}} \loss_e(\bar{\modifier} \circ \represent)
        \label{eq:IRM_orig}
    \end{align} 
where $\modifier$ is assumed to be unrestricted in the space of all measurable functions.

% {\blue
%     Given a predictor  $\modifier \circ \represent: \Inp \to \Out$,}
%     % $\predictor \equiv \modifier \circ \represent$,
% 	the feature representation $\represent: \Inp \to \Latent$ is \emph{IRM invariant} over a set of environments $\mathcal{E}$ 
% 	if there exists $\modifier : \Latent \to \Out$ such that $\modifier$ is simultaneously optimal {\cyan on $\represent$} for all environments in $\mathcal{E}$,
% 	{\it i.e.}%
% % 	\footnotemark
% 	% 	
% 	\begin{align}
% 		\forall e \in \mathcal{E}, ~~~~~~~~~~~~~~~~~~ 
% 		\modifier \in \argmin_{\bar{\modifier}} \loss_e(\bar{\modifier} \circ \represent)
% 		\label{eq:IRM_orig}
% 	\end{align}	
% where $\modifier$ is assumed to be unrestricted in the space of all measurable functions.
\end{definition}
% 
% \footnotetext{
% 	{\red In the original definition of IRM,} 
% 	$\modifier$ is considered to be a part of the predictor, such that $\predictor = \modifier \circ \phi$
% 	with $\phi: \Inp \to \Hidden$ and $\modifier: \Hidden \to \Out$. 
% 	{\cyan but this distinction has no real appreciable difference.}
% 
\begin{lemma}{\citep{kamath2021does}}
\label{lem:IRM}
{\blue For standard loss functions\footnotemark,} 
Definition \ref{def:IRM} is equivalent to 
\begin{align}
	\exists \modifier, ~
	\forall e \in \mathcal{E} ,	 ~~~~~~~~~~~~~~~~~~
	&
	 \expect_{e}[\labl | \represent(\inp)] = \sigma(\modifier \circ \represent(\inp)).
	%  \expect_{e}[\labl | \out] = \sigma(\out).
% 	& ~~~~~~~~
    % , 	\forall \out \in \mathcal{\Out}
	\label{eq:IRM_EIC}
\end{align}
where 
% $\modifier^*$ is  optimal on $\represent$, 
$\expect_e[\labl | \represent(\inp)] \equiv \expect_{P_{e}( \labl |\represent(\inp))}[\labl ]$
is the feature-conditioned label expectation,
and $\sigma$ is a monotonic % an invertible 
function that depends on the loss function.  %$^\text{\ref{sigma}}$. 
\footnotetext{
\label{sigma}
	% \cite{kamath2021does} showed this result for the case of 
	% {\cyan The} square loss and {\cyan the} BCE loss 
	Square loss and BCE loss
	are considered:
	$\sigma(\out)=\out$ for square loss, and $\sigma(\out)=\tanh(\out/2)$ for BCE loss.
	% 
	% Square loss
	% $g = [\partial_\out l; \partial_\labl l] = [\out-\labl; \labl-\out]$
	% $H = [\partial_{\out\out} l, \partial_{\out\labl} l; \partial_{\labl\out} l, \partial_{\labl\labl} l] = [1, -1; -1,  1]$
% 
	% logistic loss
	% $g = [\partial_\out l; \partial_\labl l] = [\sigma(\out)-\labl; -\out]$
	% $H = [\partial_{\out\out} l, \partial_{\out\labl} l; \partial_{\labl\out} l, \partial_{\labl\labl} l] = [\sigma'(\out), -1; -1,  0]$
}
% See \cite{kamath2021does} for proof.
\end{lemma}

\subsubsection{Variational formulation}

The original formulation above %Definition~\ref{def:IRM} 
is overly complex due to 
the composite predictor $\modifier \circ \represent$ and
% the decomposition of the predictor into $\represent$ and $\modifier$ and the imposed optimality condition on $\modifier$.
the optimality condition on $\modifier$.
For further analysis, we introduce 
% {\green provide/propose} the following reformulation of IRM 
the following variational formulation.
% of IRM  defined by the following {\blue variational principle}.
% simplify the definition
% {\cyan by absorbing $\modifier$ into the predictor $\predictor$ and considering arbitrary perturbations on the output,}
% which yields the following reformulation of IRM. 
% 
% 
\begin{definition}%[IRM Invariance re-expressed] %by output perturbation]
	A predictor $\predictor: \Inp \to \Out$ is {\it {\red IRM} invariant}
	over a set of environments $\mathcal{E}$,
% 	if it is simultaneously optimal for all $e \in \mathcal{E}_\text{tr}$, 
    if the risk $\loss_e(\predictor)$ 
    {\blue remains stationary} under arbitrary infinitesimal perturbations on {\green the predictor output} $\out = \predictor(\inp)$
	for all environments in $\mathcal{E}$,
{\it i.e.}
\begin{align}	
	% \forall \,\delta \modifier ,	~
	\forall e \in \mathcal{E}, 
	% ~~~~~	\delta  \loss_e(f)  & = 0, 
	% \label{eq:IRM_loss_perturb}
	% \\
	% \text{where}
	~~~~~	\delta  \loss_e(f)  
	% & \equiv {\blue \lim_{\epsilon\to 0} (\loss_e( (I + \epsilon \delta \modifier) \circ f) - \loss_e(f) )/\epsilon } \nonumber \\
	% & =  {\blue \lim_{\epsilon\to 0} (\loss_e( f + \epsilon \delta \modifier \circ f) - \loss_e(f) )/\epsilon 	} \nonumber \\
	& = \lim_{\epsilon\to 0}
	\expect_{P_e(\out, \labl)}  [l(\out+\epsilon \delta\modifier(\out),\labl)-l(\out,\labl)] / \epsilon  \nonumber \\ 
	% \expect_{P_e(\out)} [  \expect_{P_e( \labl | \out)}  [l(\out+\epsilon\delta\modifier(\out),\labl)-l(\out,\labl)] ]/ \epsilon  \nonumber \\ 
	& = \expect_{P_e(\out)} [  \expect_{P_e( \labl | \out)} [ \partial_\out l(\out,\labl)] \cdot  \delta\modifier(\out) ] 
	= 0 
	\label{eq:IRM_loss_perturb}
\end{align}
where 
% $I: \Out \to \Out$ is the identity map and 
% 
$\delta\modifier: \Out \to \Out$
is an arbitrary perturbation that is unrestricted in the space of all measurable functions,
{\blue and $\delta  \loss_e(f) $ denotes the resulting change in risk.}%
% and $\partial_\out$ denotes the partial derivative w.r.t. $\out$.
% 
\label{def:IRM2}
\end{definition}
% 
% This definition is motivated by the following observation of \cite{arjovsky2019invariant}, which corresponds more closely to an intuitive definition of invariance.
% Because {\blue unrestricted} optimal {\red classifiers} {\purple under standard loss functions can be realized via a conditional label distribution, {\it i.e.} $\modifier^*(\out) = \expect[\labl|\out]$,} 
% 
% 
\begin{lemma}%[{\red EIC}]
\label{lem:IRM2}
{\blue For standard loss functions$^\text{\ref{sigma}}$,} Definition~\ref{def:IRM2} % The variational IRM 
is equivalent to 
% eq~\eqref{eq:IRM_EIC} of the original IRM
% {\red with $\modifier$ absorbed into $f$. That is, considering the predictor's output $\out = \predictor(\inp)$ as the feature representation,}
% which reveals a more intuitive notion of invariance, {\cyan called Environment Invariant Constraint (EIC)}. 
% 
\begin{align}
	\forall e \in \mathcal{E} ,	 ~~~~~~~~~~~~~~~~~~
	&
	 \expect_{e}[\labl | \out] = \sigma(\out),
% 	& ~~~~~~~~
    % , 	\forall \out \in \mathcal{\Out}
	\label{eq:IRM_var_EIC}
\end{align}
{\blue which is equivalent to % a simplified version of 
Lemma~\ref{lem:IRM} with %$\predictor$ replacing 
the composite predictor $\modifier \circ \represent$
replaced by $\predictor$.}
% 
% which yields
% % 
% \begin{align}
% 	\forall e_1,e_2 \in \mathcal{E} ,	 ~~~~~~~~~~~~~~~~~~
% 	&
% % 	\expect_{e_1}[\labl|\predictor(\inp)] = \expect_{e_2}[\labl|\predictor(\inp)]
% 	 \expect_{e_1}[\labl | \out] =  \expect_{e_2}[\labl | \out].
% % 	& ~~~~~~~~
%     % , 	\forall \out \in \mathcal{\Out}
% 	\label{eq:IRM_var_EIC2}
% \end{align}
\end{lemma}
\begin{proof}
Eq~\eqref{eq:IRM_loss_perturb} is satisfied if and only if 
% Because  optimal classifiers (or regressors) under standard loss functions can be realized via a conditional label distributions ($f^*(\inp) = \expect_e[\labl | x]$), an invariant representation must satisfy the following Environment Invariance Constraint (EIC):
% $\forall e \in \mathcal{E}$, 
$\expect_{P_e( \labl | \out)} [ \partial_\out l(\out,\labl)]= 0$.
% Note that
% and $\partial_\out l(\out, \labl) =  - \labl + \sigma(\out) $$^\text{\ref{sigma}}$
For standard loss functions, %$^\text{\ref{sigma}}$, 
the loss derivative has the form 
$ \partial_\out l(\out,\labl)=- \labl + \sigma(\out) $, which yields
$\expect_{e} [ \partial_\out l(\out,\labl) | \out]= - \expect_e[ \labl | \out] + \sigma(\out)=0$.

% - \labl + \sigma(\out) $
% Note that
% and $\partial_\out l(\out, \labl) =  - \labl + \sigma(\out) $$^\text{\ref{sigma}}$.%
% 
% $\partial_\out l_\text{sq}(\out, \labl) = \out - \labl$
% and $\partial_\out l_\text{log} (\out, \labl) = \tanh(\out/2) - \labl$.
% 
% Therefore, 
% $ 
% % \expect_{P_e( \labl | \out)} [ \partial_\out l(\out,\labl)]  =  
% - \expect_e[ \labl | \out] + \sigma(\out)= 0.$ %$^\text{\ref{sigma}}$. 
% 
% we arrive at
% \begin{align}
%     \label{eq:IRM_EIC2}
% 	\forall e \in \mathcal{E} ,	 ~~~~~~
% 	&
%     \expect_{P_e( \labl | \out)} [\sigma(\out) - \labl ] 
%     = \sigma(\out) - \expect_e[ \labl | \out] = 0.
%     % & ~~~~~~~~
% \end{align}
% 
\end{proof}
{\blue 
% \footnote{
Note that 
even though Definition~\ref{def:IRM2} 
describes only the first-order condition for % is equivalent to 
the predictor to be simultaneously optimal over all environments,
%can be seen as 
% describes the first order optimality condition of Definition~\ref{def:IRM}. 
this is indeed the necessary and sufficient condition for optimality, since the loss function $l$ is convex.
% with respect to $\out$.
% 
% Note that 
This yields  %the above variational 
% simplifies the IRM formulation 
% yields the equivalent description of invariance as the original IRM formulation 
a simpler formulation of IRM without requiring a composite form for the predictor $\modifier \circ \represent$.
}

\subsubsection{Conservation law of IRM}

{\blue
As noted in 
\cite{arjovsky2019invariant,kamath2021does},  % observed that
the essence of IRM's invariance %principle	
is the conservation of the \emph{feature-conditioned label expectation},} % across environments,
i.e. % Environment Invariant Condition (EIC) condition 
\begin{align}
	\forall e_1,e_2 \in \mathcal{E} ,	 ~~~~~~~~~~~~~~~~~~
	&
% 	\expect_{e_1}[\labl|\predictor(\inp)] = \expect_{e_2}[\labl|\predictor(\inp)]
\expect_{e_1}[\labl | \represent(\inp)] =  \expect_{e_2}[\labl | \represent(\inp)].
% \expect_{e_1}[\labl | \out] =  \expect_{e_2}[\labl | \out],
% 	& ~~~~~~~~
    % , 	\forall \out \in \mathcal{\Out}
	\label{eq:IRM_True_EIC}
\end{align}
% eq~\eqref{eq:IRM_EIC}. 
% This conservation law
% which {\green describes} the {\red very essence} of IRM's invariance.
% {\green which contains/described by} 
% which is described by $|\mathcal{E}|-1$ equality {\green conditions}.}
% 
% 
{\green 
% Eq~\eqref{eq:IRM_EIC}\eqref{eq:IRM_var_EIC}
This result can be easily seen from eq~\eqref{eq:IRM_EIC},\eqref{eq:IRM_var_EIC},
% indeed implies the conservation law, %eq~\eqref{eq:IRM_True_EIC},  
% This conservation law is implied by eq~\eqref{eq:IRM_var_EIC},
% as well as by eq~\eqref{eq:IRM_EIC}, % (also by eq~\eqref{eq:IRM_EIC}),
since their %whose 
RHS term  $\sigma(\out)$ is constant with respect to the environment index $e$.

{\red 
\paragraph{Remark}
Notice %Note that 
%, however, %we note that 
% there exists 
an intriguing discrepancy in the number of constraints:} 
IRM (eq~\eqref{eq:IRM_orig},\eqref{eq:IRM_EIC},\eqref{eq:IRM_var_EIC})
imposes one constraint per environment,
total of $|\mathcal{E}|$  constraints, 
% , therefore 
whereas % However, there is a critical difference: 
the conservation law describes 
% eq~\eqref{eq:IRM_True_EIC}  that %requires %the conditional expectation 
% $ \expect_{e}[\labl | \out]$ must share the same value across environments, which is 
% only
$|\mathcal{E}|-1$ equality relationships 
that $ \expect_{e}[\labl | \out]$ should share the same value
% that the value of $ \expect_{e}[\labl | \out]$ is shared
across environments. %constraints.
% whereas eq~\eqref{eq:IRM_var_EIC} (also eq~\eqref{eq:IRM_EIC})
%  imposing $|\mathcal{E}|$ constraints.
% 
The missing constraint 
% This is because
% of course
% hidden not shown 
% in eq~\eqref{eq:IRM_True_EIC} 
is that 
IRM additionally requires %ment that 
the shared value of $ \expect_{e}[\labl | \out]$ %in eq~\eqref{eq:IRM_True_EIC} 
% is additionally required by IRM 
to also be equal to $\sigma(\out)$.
This discrepancy %remarks 
% discrepancy % remark {\red hints/reveals/shows/reflects}
{\green emphasizes} %/points at} 
the fact %a crucial remark % revelation  % IRM attains 
that the equality relationships in 
eq~\eqref{eq:IRM_True_EIC} %is 
are not direct constraints imposed to hold {\red between} environments, 
% directly imposed by IRM, 
% are not constraints imposed {\red between} environments, 
% 
but rather 
{\red a byproduct --- an indirect consequence}   
of separate individual constraints 
% {\cyan for each environment} 
all sharing a common intermediate term, $\sigma(\out)$.
Furthermore, %Moreover, 
it shows that 
the invariance guarantee of IRM singularly depends on the constancy of this shared term across environments,
%We show below that 
which
% {\cyan, as we show in the next section,} %we show that 
% this indirect mechanism of attaining invariance
{\red 
proves to be a single point of failure for IRM.} % which breaks in IRM-v1.}
% \cyan as we show in the next section.
}

% {\blue This extra constraint 
% reflects %is due to %reflects %inevitably arises % originates from the fact that 
% the indirect way IRM attains the conservation, % in an indirect manner,
% % the conservation law is imposed in an indirect manner, 
% {\it i.e.} by {\green equating} $ \expect_{e}[\labl | \out]$ to %be equal to 
% an intermediate term $\sigma(\out)$, %, {\cyanbut ultimately irrelevant,}
% % for each environment,
% which imposes one constraint per environment.}

% additionally imposes that value to be also equal to the RHS term $\sigma(\out)$.

% {\cyan This extra constraint, while irrelevant to the invariance {\green principle/condition}, may seem entirely harmless. }

% which  does not contribute to invariance. %This extra constraint 
% which is irrelevant to the invariance condition.
% {\cyan (for each value of $f(\inp)$)}.}
% {\cyan In following sections, we analyze the unexpected consequences of this extra constraint.}
% However, it contains an extra constraint compared to eq~\eqref{eq:IRM_True_EIC} that was not intended for.

% In the following sections, we show that this extraneous constraint has an unexpected consequence.

\subsubsection{IRM-v1}
\label{sec:IRMv1}
Due to the impracticality  %{\green practical limitations} 
of considering the unrestricted function space of $\modifier$, \cite{arjovsky2019invariant} suggested restricting $\modifier$ to the space of linear functions.
% {\cyan hoping that it would yield a necessary condition for IRM invariance.}
% In our formulation, an equivalent restriction %is  to considering 
In the variational formulation, this corresponds to restricting the output perturbations $\delta \modifier$ to the space of linear functions, 
which, for scalar outputs, is equivalent to the identity function, $\delta \modifier (\out) = \out$.
% 
 % $\modifier (\out) = \alpha \out$.
 % $\delta\modifier (\out) = (\alpha) \out$.
 % 
This reduces eq~\eqref{eq:IRM_loss_perturb} to
\begin{align}
	\label{eq:IRM_loss_perturb_linear}
	\forall e \in \mathcal{E}, 
	~~~~~	\delta  \loss_e(f)  
    % &= \expect_{P_e(\out)} [  \expect_{P_e( \labl | \out)}  [ \, \partial_\out l(\out,\labl) \cdot \out \, ]]  
	&= \expect_{e} [ \, \partial_\out l(\out,\labl) \cdot \out \,] % {\blue \alpha}
	% \nonumber \\
    % & = \expect_e [  \out\sigma(\out) ]  - \expect_e [  \out \labl  ] 
	= 0 .
\end{align}
%  identical to IRM-v1
% {\blue
The reduced constraints %This reduced formulation 
in eq~\eqref{eq:IRM_loss_perturb_linear}
are identical to %called 
IRM-s in \cite{kamath2021does}.%
\footnote{IRM-s imposes
% which uses the constraints  % {\green defines $\delta  \loss_e(f)$ as/
% $\forall e \in \mathcal{E}, 
% \delta  \loss_e(f) = 
$\partial_\modifier 
\expect_{e}  [ \, l( \modifier \cdot \out,\labl) ] 
% \expect_{P_e(\out)} [  \expect_{P_e( \labl | \out)}  [ \, l(\modifier \cdot \out,\labl) ]] 
= 0$,
where $\modifier$ is a scalar factor.
Note that $\partial_\modifier l(\modifier \cdot \out,\labl) = \partial_\out l(\out,\labl) \cdot \out$.}
In \cite{arjovsky2019invariant}, this reduced formulation is termed IRM-v1 when the constraints are imposed in a {\it soft manner}, {\it i.e.} as squared penalty terms (See eq~\eqref{eq:penalty_loss}).
In the literature, the term IRM-v1 is widely used for the reduced formulation regardless of whether hard or soft constraints are used, which we adopt here. % as well. 
% These reduced constraints are also used It is also called IRM-v1
% }
% {\red [This is identical to IRM-v1 in \cite{arjovsky2019invariant}
% Although \cite{arjovsky2019invariant} went through a much more convoluted derivation.]}

% Note that the linear restriction would reduce eq~\eqref{eq:IRM_EIC} yields 
% $\expect_{P_e(\out)} [ \expect_e[\labl | \out] \cdot \out]  = \expect_e [  \out \labl  ]$ 
% to be conserved across environments, 
% which is indeed a necessary condition for the original {\green conservation law/invariance} eq~\eqref{eq:IRM_EIC}. 
% 
% {\cyan Note that 
% % In eq~\eqref{eq:IRM_loss_perturb_linear}, however,  
% % {\cyan if the feature distribution $P_e(\out)$ were constant}
% the first term $\expect_e [  \out\sigma(\out) ] $ of eq~\eqref{eq:IRM_loss_perturb_linear} originates from the RHS of eq~\eqref{eq:IRM_var_EIC}.}
% of eq~\eqref{eq:IRM_loss_perturb_linear} 
% {\red If this term were constant, it would yield the conservation of 
% $ \expect_e [  \out \labl  ] = \expect_{P_e(\out)} [ \expect_e[\labl | \out] \cdot \out] $,
% a necessary condition for the {\green conservation} of $\expect_e[\labl | \out]$ in eq~\eqref{eq:IRM_var_EIC}
% [THIS IS WRONG].}
% 

{\blue
% Despite the practical applicability, 
However,
% it has been empirically shown that
IRM-v1 %(and IRM-s) 
has been empirically found 
% that IRM-v1 %(and IRM-s) 
to behave quite differently from IRM and fail even in simple problems
\citep{kamath2021does}. 
% {\red Here, we can explain this failure mode analytically.}
This failure mechanism can be analytically understood here:
Since %Note that 
$ \partial_\out l(\out,\labl) \cdot \out  = 
 \out ( \sigma(\out) - \labl) $
for standard loss functions$^\text{\ref{sigma}}$,
eq~\eqref{eq:IRM_loss_perturb_linear} %IRM-v1 
% For each environment $e$,
is equivalent to % {\green requires/imposes} 
\begin{align}
	\label{eq:IRM_v1_EIC}
	\forall e \in \mathcal{E}, ~~~~~~
	\expect_e [  \out \labl  ] = \expect_e [  \out\sigma(\out) ] .
 \end{align}
% 
% For each environment $e$, eq~\eqref{eq:IRM_loss_perturb_linear} $ \expect_e [  \out \labl  ]$ to have the same value as $\expect_e [  \out\sigma(\out) ] $.
% {\green
% % Note that LHS of eq~\eqref{eq:IRM_v1_EIC}
% % $\expect_e [  \out \labl  ] $ 
% % is the $\expect_{e}[\labl | \out] $
% which is the expected 
% eq~\eqref{eq:IRM_var_EIC}, 
% 
Note that 
{\green the RHS %$\expect_e [  \out\sigma(\out) ] $ 
of eq~\eqref{eq:IRM_v1_EIC}
 originates from the RHS of %$\sigma(\out)$ in 
eq~\eqref{eq:IRM_var_EIC}.
% In eq~\eqref{eq:IRM_v1_EIC}, 
Unlike  $\sigma(\out)$  of eq~\eqref{eq:IRM_var_EIC}, %however, 
% whose constancy plays a singular role in mediating the conservation law eq~\eqref{eq:IRM_True_EIC}, %relationship.}
%  
however, 
$\expect_e [  \out\sigma(\out) ]$ %of eq~\eqref{eq:IRM_v1_EIC} 
is not constant, 
since it involves expectation that depends on the environment,
and therefore it fails to mediate %prescribe 
any meaningful invariance relationship. % any more.
{\red 
% While this inconsistency was empirically shown in \cite{kamath2021does}, the {\red fundamental} cause of the problem was unknown. 
% In Supplementary Materials B, we include analysis of a more general type of restriction.
% 
% While the above analysis only considered  the special case of linear function space for perturbations, 
In Supplementary Materials, we generalize this result to 
% of this failure mechanism 
% The result generalizes to 
% applies to a more general 
the wider class of perturbations 
% when 
% even 
% when the perturbations 
that are %restricted to 
mixtures of nonlinear basis functions
% to which the linear case above is 
% generalized for a wider case of restrictions 
% ,
% whenever a restriction is imposed on the function space of $\delta\modifier$ (or $\modifier$). 
\cyan (See Supplementary Materials B).
}

\paragraph{Fundamental flaw of IRM}
The above analysis shows that 
IRM's indirect mechanism %for invariance
for attaining invariance through a shared intermediate term 
% the indirect mechanism of IRM 
% that depends on a shared intermediate term 
% Therefore, restricting the function space of $\delta\modifier$ (or $\modifier$) 
% {\green Note that this holds true for a general causal graph and not just for the causal graph considered in Fig~\ref{fig:sem}.}
% the indirect way of imposing invariance in IRM 
is in fact quite fragile, %a weak link 
which easily breaks %can easily fail %breaks  %fails 
when the function space of $\delta\modifier$ (or $\modifier$) gets restricted.
{\red We identify this as the} {\it fundamental {\cyan design} flaw of IRM}.
% 

% is considered, this
% would result in some form of expectation over eq~\eqref{eq:IRM_var_EIC}, in which 
% an environment dependency of RHS is unavoidable. 
% 

}}

\subsection{MRI Paradigm}

We now introduce a {\green complementary notion} of invariance by 
% modifying definition~\ref{def:IRM2} to use 
considering infinitesimal perturbations on label,
% based on label perturbation, 
which we call the \emph{\green Mirror Reflected IRM}, or \emph{MRI}.
\begin{definition} %[MRI invariance]
	\label{def:MRI_invariance}
	A predictor $\predictor: \Inp \to \Out$ is {\it {\red MRI} invariant} over a set of environments $\mathcal{E}$,
% 	if  arbitrary {\blue infinitesimal} label perturbation equally affects all risk $\loss_{e}(f)$ 
% if the change in the risk due to a an arbitrary {\blue infinitesimal} label perturbation is \emph{\red identically shared} across all environments, 
	% if for all environments $\loss_{e}(f)$ gets \emph{\red identically affected} by arbitrary {\blue infinitesimal} perturbations on label, 
	if the change in risk %$\loss_{e}(f)$ 
	due to arbitrary {\blue infinitesimal} label perturbations is \emph{shared} across environments, 
	% for all environments, 
	{\it i.e.}
	\begin{align}
		\label{eq:MRI_constraint}
		% \forall \, \delta \modifier, ~	
		\forall e_1,e_2 \in \mathcal{E},	 ~~~~~~   \delta  \loss_{e_1}(f) & =  \delta \loss_{e_2}(f) \\
% 		\label{eq:perturbed_loss}
		\text{where} 	 ~~~~~~
		\delta  \loss_e(f) 
		& \equiv \lim_{\epsilon\to 0} 
% 		\expect_{P_e(\inp,\labl)} [l(f(\inp),\labl+\epsilon\delta\modifier(\labl))-l(f(\inp),\labl)]/\epsilon
		% \expect_{P(\labl)} [  \expect_{P_e(\out | \labl)}  [ \, l(\out,\labl+\epsilon\delta\modifier(\labl))-l(\out,\labl) \, ]]/\epsilon \nonumber \\
		\expect_{P_e(\inp,\labl)}  [ \, l(\out,\labl+\epsilon\delta \modifier(\labl))-l(\out,\labl) \, ]/\epsilon \nonumber \\
		& = \expect_{P(\labl)} [ \, \expect_{e}  [ \, \partial_\labl l(\out,\labl) | \labl \,] \cdot \delta \modifier(\labl)\, ]
		\label{eq:MRI_constraint2}
	\end{align}
    where $\delta \modifier: \Labl \to \Labl$ 
    is an arbitrary perturbation that is unrestricted in the space of all measurable functions,
	{\blue and $\delta  \loss_e(f) $ denotes the resulting change in risk.
	$ \expect_{e} [  \partial_\labl l(\out, \labl) |  \labl ] 
	\equiv \expect_{P_e(\out | \labl)}  [  \partial_\labl l(\out,\labl)] $.

% and $\partial_\out$ denotes the partial derivative w.r.t. $\out$.
}
\end{definition}
% 
% is a {\red modifier function} that {\cyan deterministically} perturbs the value of label.

\begin{lemma} %[{\red EIC}]
% This notion of invariance 
{\blue For standard loss functions,}
Definition~\ref{def:MRI_invariance}
is equivalent to 
% {\it i.e.}
\begin{align}
	\forall e_1,e_2 \in \mathcal{E} ,	 ~~~~~~~~~~~~~~~
	&
	\expect_{e_1}[\out|\labl ] =  \expect_{e_2}[\out|\labl],
% 	&
	\label{eq:MRI_conservation}
\end{align}
which
conserves the label-conditioned feature expectation
$\expect_e[\out|\labl] \equiv \expect_{P_{e}( \out|\labl)}[\out]$
across environments.

% where $\expect_e[\out|\labl] \equiv \expect_{P_{e}( \out|\labl)}[\out]$.
\end{lemma}

\begin{proof}
% 	{\cyan Note that $P(\labl)$ is shared across environments,}
    Eq~\eqref{eq:MRI_constraint}  is satisfied if and only if 
	$ \expect_{e} [  \partial_\labl l(\out, \labl) |  \labl ] $
	% $ \expect_{P_e(\out | \labl)}  [ \, \partial_\labl l(\out,\labl)] $
	% $ \expect_{e_1} [  \partial_\labl l(\out, \labl) |  \labl ] = \expect_{e_2} [  \partial_\labl l(\out, \labl) |  \labl ], ~  \forall e_1,e_2 \in \mathcal{E}$.
	is conserved.
	For standard loss functions,
	% For square loss and the lostic loss, 
	the loss derivative has the form
    $\partial_\labl l(\out, \labl) = - \out + \rho(\labl)$%
	\footnotemark. %$^\text{\ref{rho}}$,
	% 
    % $\partial_\labl l_\text{sq}(\out, \labl) = - \out + \labl$
	% and 
	% $\partial_\labl l_\text{log} (\out, \labl) % 	{\cyan = \log(\sigma(\out)/(1-\sigma(\out))) } 
	% = -\out$,
    % this yields eq~\eqref{eq:MRI_conservation}.
	Therefore, 
	$ \expect_{e_1} [  \partial_\labl l(\out, \labl) |  \labl ] - \expect_{e_2} [  \partial_\labl l(\out, \labl) |  \labl ]
	= \expect_{e_1} [ \out |  \labl ] - \expect_{e_2} [  \out  |  \labl ] = 0$.
	\footnotetext{
		\label{rho}
		% {\cyan The} square loss and {\cyan the} BCE loss 
		Square loss and BCE loss
		are considered:
		$\rho(y)=y$ for square loss, and  $\rho(y)=0$ for BCE loss.}	
\end{proof}
{\green
\paragraph{Remark}
Note that 
% {\cyan unlike IRM,}
MRI attains invariance in a direct manner without involving any intermediate term, which 
results in the number of constraints
% imposing the correct number of constraints 
in eq~\eqref{eq:MRI_constraint} matching  the conservation law eq~\eqref{eq:MRI_conservation}.}
% eq~\eqref{eq:MRI_constraint} imposes $|\mathcal{E}| -1$ constraints

% {\cyan [More generally, for any loss $l(\out, \labl) $ that is convex with respect to $\out$, MRI-v1 yields conservation of $\expect_{e}[\sigma(\out)|\labl ] $ where $\sigma(\out)$ is a monotonic function of $\out$. ]}

% \subsubsection{IRM-r: relaxed IRM}

% Note that MRI \eqref{eq:MRI_constraint} 
% {\red uses the constraint} that the perturbed loss being shared across environments ($N-1$ constraints)
% whereas IRM \eqref{eq:IRM_loss_perturb}
% makes a further constraint that the shared perturbed loss is also equal to zero ($N$ constraints). {\red [define $N$ to be the total number of environments]}

% {\red
% What if we relax the zero value constraint on IRM?

% % $\delta\loss_{e_1}=\delta\loss_{e_2}$
% This wouldn't work on IRM, 
% because $P_{e}(o)$ is not {\green preserved/the same} across environments. 
% (but we can still try using it and find out how badly it works.)
% }

% \subsection{Linear versions}
% Kamath showed that  $\mathcal{I} \in \mathcal{I}_s = \mathcal{I}_\text{lin}$.

\subsubsection{MRI-v1}

Restricting the label perturbations to the space of {\green linear functions}, or equivalently, 
an identity function $\delta \modifier (\labl) = \labl$,
reduces % the MRI formulation 
eq~\eqref{eq:MRI_constraint2}  to 
%
% \begin{align}
$    \delta\loss_{e}(f) 
% {\cyan = \expect_{P(\labl)} [ \expect_{P_{e}(\out | \labl)} [ \, \partial_\labl l(\out, \labl) ] \, \labl ] }
% = \expect_{P_e(\out , \labl)} [ \, \partial_\labl l(\out, \labl)  \cdot \labl \, ],
= \expect_{e} [ \, \partial_\labl l(\out, \labl)  \cdot \labl \, ]$.
%     \label{eq:MRI_constraint_linear}
% \end{align}
This reduces the conservation law eq~\eqref{eq:MRI_conservation} to
% prescribes the conservation of label$\cdot$feature expectation, {\it i.e.}
\begin{align}
	\forall e_1,e_2 \in \mathcal{E} ,	 ~~~~~~~~~~~~~~~~~~
% 	&
	\expect_{e_1}[\out\labl ] =  \expect_{e_2}[\out\labl].
% 	&
	\label{eq:MRIv1_conservation}    
\end{align}
which describes %is in fact %Note that this is 
a necessary condition for the MRI invariance eq~\eqref{eq:MRI_conservation}.
% {\red Note that this is {\green also} the necessary condition for IRM invariance eq~\eqref{eq:IRM_True_EIC}. \cyan IRM-v1 eq~\eqref{eq:IRM_loss_perturb_linear} if the first term $\expect_e [  \out\sigma(\out) ] $ were indeed constant across environments.}
% 
{\blue 
Therefore, 
% unlike IRM's invariance, 
MRI's direct mechanism for attaining invariance %conservation law 
continues to hold %remains valid % survives %involving an intermediary term fails 
when $\delta\modifier$ is restricted to the space of linear functions.
}

% {\cyan Moreover, eq~\eqref{eq:MRIv1_conservation} describes linear equality constraints with respect to the output $\out$. The consistency with MRI invariance, as well as the convexity of constraint, provides MRI-v1 with  practical advantages over IRM-v1, which we show in following sections.}

% {\cyan $\mathcal{I}_{IRM} \subset \mathcal{I}_{s}$. This is a non-convex constraint, which causes many problems for IRM. }

% \subsection{Training Objective}

\input{method}

%% file: tikz_figure.tex
\begin{wrapfigure}{r}{0.4\textwidth}
    \begin{center}
    %  \hspace*{\fill}
    %  \begin{subfigure}[b]{0.3\textwidth}
    %      \centering
        \begin{tikzpicture}[
                > = stealth, % arrow head style
                shorten > = 1pt, % don't touch arrow head to node
                auto,
                node distance = 2cm, % distance between nodes
            ]
            \node[obs] (y) {$Y$};
            \node[state] (z_o) [above left of=y] {$\Latent_\causal$};
            \node[state] (z_e) [above right of=y] {$\Latent_\spu$};
            \node[obs] (x) [above=2cm of y] {$\Inp$};
            \node[obs] (e) [right=2cm of y] {$\Env$};
        
            \path[->] (y) edge node {} (z_o);
            % \path[->] (z_o) edge node {} (y);
            \path[->] (y) edge node {} (z_e);
            \path[->] (e) edge node {} (z_e);
            % \path[->] (e) edge node {} (y);
            \path[->] (z_o) edge node {} (x);
            \path[->] (z_e) edge node {} (x);
        \end{tikzpicture}
    \end{center}
    \caption{\blue Causal graph % A Bayesian network 
    depicting the data generating process. Shading indicates the variable is observed.} 
    \vspace{-0.4cm}
    \label{fig:sem}
\end{wrapfigure}
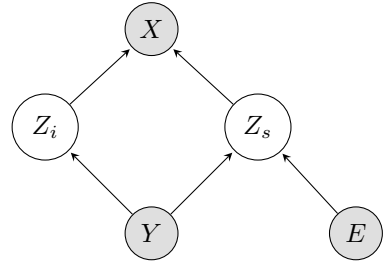

% \begin{figure}{r}{0.4\textwidth}
%     \begin{center}
%     %  \hspace*{\fill}
%     %  \begin{subfigure}[b]{0.3\textwidth}
%     %      \centering
%         \begin{tikzpicture}[
%                 > = stealth, % arrow head style
%                 shorten > = 1pt, % don't touch arrow head to node
%                 auto,
%                 node distance = 2cm, % distance between nodes
%             ]
%             \node[obs] (y) {$Y$};
%             \node[state] (z_o) [above left of=y] {$\Latent_\causal$};
%             \node[state] (z_e) [above right of=y] {$\Latent_\spu$};
%             \node[obs] (x) [above=2cm of y] {$\Inp$};
%             \node[obs] (e) [right=2cm of y] {$\Env$};
        
%             \path[->] (y) edge node {} (z_o);
%             % \path[->] (z_o) edge node {} (y);
%             \path[->] (y) edge node {} (z_e);
%             \path[->] (e) edge node {} (z_e);
%             % \path[->] (e) edge node {} (y);
%             \path[->] (z_o) edge node {} (x);
%             \path[->] (z_e) edge node {} (x);
%         \end{tikzpicture}
%     \end{center}
%     \caption{Causal graph}
% \end{figure}

%% file: method.tex
\begin{table}[t!]
	\begin{center}
	  \caption{
		% {\red Summary of invariance constraints.}
	{\blue 	Constraint functions of IRM-v1 and MRI-v1.
    %  used by different {\green methods/algorithms}. 
    % $\vec{\delta \loss} \in \mathbb{R} ^{|\mathcal{E}_\text{tr}|}$
    % is a vector of perturbed risks $\delta \loss_e$ for $e\in \mathcal{E}_\text{tr}$.
    % $Q$ is an orthonormal matrix that satisfies $Q \, \vec{1} = \vec{0}$.
    % % $Q Q^\intercal  = I^{|\mathcal{E}_\text{tr}|-1}$.
    $\out = \predictor(\inp; w)$. 
    }
	%   
	% {\cyan Note that IRM and MRI as shown just as a reference. They are not actually used in practice.}
	%   $^\text{\ref{sigma}\ref{rho}}$
	  }  %
	  \label{tab:all_constraints}
	  \vspace*{3mm}
	  \begin{tabular}		{l | r | c} 		% {M{1.5cm}|M{2.5cm} |M{6.5cm}} 
		\toprule
		% & $\delta\loss_{e}(f)$
		Algorithm & Constraint $ \vec{c}(f)$ & $\delta \loss_e(\predictor)$ \\
		\midrule
		% IRM    
		% & $ \vec{\delta \loss}(f) ~~~~~ $
		% & $ \expect_{e} [  \partial_\out l \,|\,  \out ] 
		% % {\cyan = \sigma(\out) - \expect_e[\labl | \out]} 
		% $ 
		% \\[0.05cm]
		% MRI    
		% & $ Q \, \vec{\delta \loss}(f) ~~~~~ $
		% & $ \expect_{e} [  \partial_\labl l \,|\, \labl ] 
		% % {\cyan  = \rho(\labl) - \expect_e[  \out | \labl] } 
		% $ 
		% \\[0.05cm]
		IRM-v1 
		& $  \vec{\delta \loss} (f) ~~~~~ $
        & $ \expect_{e} [ \, \partial_\out l(\out, \labl)  \cdot \out \, ] $
		% & $ \expect_{e} [  \partial_\out l \cdot \out ] 
		% % {\cyan  =  \expect_e[\out \sigma(\out) - \out\labl]} 
		% $ 
		\\[0.05cm]
		MRI-v1 
		& $ Q\, \vec{\delta \loss}(f) ~~~~~ $
		& $ \expect_{e} [ \, \partial_\labl l(\out, \labl)  \cdot \labl \, ] $
        % $  \expect_{e} [  \partial_\labl l \cdot \labl] 
		% % {\cyan  = \expect_e[\labl\rho(\labl) -  \out \labl ] } 
		% $
		\\[0.05cm]
		\bottomrule
	  \end{tabular}
	\end{center}
\end{table}

% \begin{table}[t!]
% 	\begin{center}
% 	  \caption{Summary of invariance constraints. $\sigma(\out) = \tanh(\out/2)$.}
% 	  \label{tab:all_constraints}
% 	  \vspace*{3mm}
% 	  \begin{tabular}{M{1.5cm}|M{2.5cm}|M{2.5cm}|M{3.5cm}} \toprule
% % 		& for general loss 
% 		& Square loss ($\delta\loss_{e}(f)$) & Logistic loss ($\delta\loss_{e}(f)$)
% 		& Constraint \\
% 		\midrule
% 		IRM    
% % 		& $\expect_{P_e( \labl | \out)} [ \partial_\out l(\out,\labl)] $ 
% 					& $ \expect_e[\labl | \out] - \out $ 
% 						& $  \expect_e[\labl | \out] - \sigma(\out)$ 
% 				& $\delta\loss_{e}(f) = 0$ \\[0.05cm]
% 		MRI    
% % 		& $\expect_{P_e(\out | \labl)}  [ \, \partial_\labl l(\out,\labl)] $ 
% 					& $\expect_e[\out | \labl] - \labl $ 
% 						& $\expect_e[\out | \labl] $ 
% 				&  $\delta\loss_{e}(f) $ is conserved \\[0.05cm]
% 		IRM-v1 
% % 		& $\expect_{P_e(\out, \labl)} [ \, \partial_\out l(\out,\labl) \cdot \out \, ]  $ 
% 					& $ \expect_e[\out \out - \out\labl] $
% 						& $ \expect_e[\out \sigma(\out) - \out\labl] $ 
% 				& $\delta\loss_{e}(f) = 0$  \\[0.05cm]
% 		MRI-v1 
% % 		& $\expect_{P_e(\out, \labl)} [ \, \partial_\labl l(\out, \labl)  \cdot \labl \, ] $ 
% 					& $   \expect_e[\out \labl  -\labl\labl]$ 
% 						& $  \expect_e[\out \labl]$ 
% 				&  $\delta\loss_{e}(f)$ is conserved  \\[0.05cm]
% 		\bottomrule
% 	  \end{tabular}
% 	\end{center}
% \end{table}

\section{Methods} 
\label{section:methods}

\subsection{Constrained optimization problem}

The full methods of IRM-v1 and MRI-v1 %involve solving 
can be formalized as a constrained optimization problem
\begin{align}
    ~~~~~~~~~~~~
	\min_{f \in \mathcal{F}} \loss_\text{tr}(\predictor) 
	~~~~~~~~~~\textrm{subject to}  ~~
	\vec{c}(\predictor) = \vec{0} ,
    \label{eq:constrained_objective}
\end{align}
% {\green involve/require} solving equality-constrained {\cyan optimization} problems given as
% {\red that minimizes}
where 
$\loss_\text{tr}%(\predictor) 
= \frac{1}{|\mathcal{E}_\text{tr}|}\sum_{e\in\mathcal{E}_\text{tr}} \loss_e%(\predictor)
$
is the average risk over the set of training environments 
$\mathcal{E}_\text{tr} \subset \mathcal{E}$.
% and $w$ is the parameters of the predictor $\predictor$.
% 
% subject to invariance constraints 
% The constraints that are used in different {\green methods/algorithms} are 
% summarized in Table~\ref{tab:all_constraints}.}
% Table~\ref{tab:all_constraints} shows the constraints used by different {\green methods/algorithms}.  
% and the {\cyan vector-valued} constraint function contains 
% the perturbed risks $\delta \loss_e(\predictor)$, ${\forall e\in\mathcal{E}_\text{tr}}$ for the case of IRM-v1 as {\green the perturbed risk itself}, or the difference between pairs of perturbed risks for the case of MRI-v1 (See Table~\ref{tab:all_constraints}). 
% 
{\blue 
The constraint functions are 
$\vec{c} = \vec{\delta \loss} \in \mathbb{R} ^{|\mathcal{E}_\text{tr}|}$ 
for IRM-v1 and 
$\vec{c} = Q\,\vec{\delta \loss} \in \mathbb{R} ^{|\mathcal{E}_\text{tr}|-1}$ 
for MRI-v1, where 
$\vec{\delta \loss} \in \mathbb{R} ^{|\mathcal{E}_\text{tr}|}$
is a vector of perturbed risks $\delta \loss_e$ for $e\in \mathcal{E}_\text{tr}$, 
and $Q \in \mathbb{R} ^{(|\mathcal{E}_\text{tr}|-1) \times |\mathcal{E}_\text{tr}|}$
is an orthonormal matrix that satisfies %is orthogonal to $\vec{1}$: {\it i.e.}
$Q \, \vec{1} = \vec{0}$,
such that $Q\,\vec{\delta \loss}$ computes the differences of 
$\delta\loss_{e}$ %perturbed risks 
between environments.
% 
% and $Q Q^\intercal  = I^{|\mathcal{E}_\text{tr}|-1} $.
% % $Q:  \mathbb{R} ^{|\mathcal{E}_\text{tr}|} \to \mathbb{R} ^{|\mathcal{E}_\text{tr}|-1 }$ 
%  $Q$ is an orthonormal matrix that satisfies $Q \, \vec{1} = \vec{0}$ and $Q Q^\intercal  = I^{|\mathcal{E}_\text{tr}|-1}$.
%  the constraint 
% ${\purple Q} \vec{c}(\predictor) = \vec{0} $
% describes $|\mathcal{E}_\text{tr}|-1$ constraints.
For example, for % in the case of 
a training set of two environments
$\mathcal{E}_\text{tr} = \{e_1, e_2\}$,  
% MRI-v1's constraint is 
$Q\,\vec{\delta \loss} = (\delta\loss_{e_1} - \delta\loss_{e_2})/\sqrt{2}$,
since $Q = [1, -1]/\sqrt{2}$.
% and $\vec{c} = [\delta\loss_{e_1}(f) ; \delta\loss_{e_2}(f) ]$.
% The constraints %of  IRM-v1 and MRI-v1 
% are summarized in
See Table~\ref{tab:all_constraints}.
}

% We consider the constrained optimization problem for an objective function $f$ subject to m constraints
% \begin{equation*}
%     c_k(w) = 0 \quad \forall k \in \{1,2, ... ,m\}.
% \end{equation*} 
% {\cyan In summary,}
% {\red the algorithms {\cyan discussed above} $\IRMorig$, $\ICMoooy$, and $\ICMoy$ 

% In this work, we consider the following two methods to solve this constrained optimization problem.

\subsection{Soft-constraint methods}

Numerically solving IRM-v1 and MRI-v1 %constrained optimization problems
requires converting the hard constraints $\vec{c}(\predictor) = \vec{0} $ to 
soft constraints,
which allows the use of off-the-shelf gradient-based optimization algorithms.

\paragraph{Penalty Method (PM)}

Penalty method is the most commonly used approach, 
including in \citet{arjovsky2019invariant},
% are relaxed into a penalty function,
which adds the squared residual constraints as a penalty term to the objective, 
% 
% to find the optima. 
% 
% is to introduce a quadratic penalty function %on the constraint 
% 
\begin{equation}
    \min_{f \in \mathcal{F}} %    \min_{w} 
    % {\cyan \lim_{\mu \to \infty}} 
    \,
    \loss_\text{tr}(\predictor)
    + \mu  \left\lVert \vec{c}(\predictor) \right\rVert ^2
% 	\sum_{e \in \mathcal{E}_\text{tr}} 
% 	\loss_e(w) + \mu  \left\lVert c_e(w) \right\rVert ^2
	\label{eq:penalty_loss}
\end{equation}
% 
% The penalty method was used in %The most common choice, as used in   .
% 
% where the parameter $\mu$ controls the contribution of the penalty term to the total objective function. 
However, this method requires 
{\blue increasing $ \mu^t \to \infty$ over training iteration $t$}
% large values for $\mu$ 
in order to approximate the exact hard-constraint,
which leads to  % suffers from 
training instability and slow convergence % ill-conditioned 
\citep{bertsekas1976multiplier}.
% 
% {\cyan 
% [Moreover, at large values of $\mu$, we often find IRM-v1 converging to the zero-predictor, especially when training deep neural network models, or when mini-batches are used, as reported by others {\purple \citep{??}}.]
% }

\paragraph{Augmented Lagrangian Method (ALM)}

ALM was introduced to overcome the limitations of penalty method \citep{bertsekas1976multiplier}, 
which adds a Lagrange multiplier term to \eqref{eq:penalty_loss}
\begin{equation}
% 	\max_{\lambda} 
    \min_{f \in \mathcal{F}} % \min_{w} 
    \,
	\loss_\text{tr}(\predictor) 
    + \mu \left\lVert \vec{c}(\predictor) \right\rVert ^2
    +  \vec{\lambda}^\intercal \cdot \vec{c}(\predictor),
	\label{eq:ALMloss}
\end{equation} 
% 
% In practice, 
% This method involves updating
where 
$\vec{\lambda}$ is typically initialized at $\vec{0}$
%  which has an annealing effect on ALM.
and %$ \vec{\lambda}$ 
updated at each training iteration $t$
% as $ \vec{\lambda}^{t+1} = \vec{\lambda}^{t} + \mu \vec{c}(\predictor(w^t)) $
to accumulate the residual constraints $\vec{c}(\predictor(w^t))$.
% and then updating $w^t$ according to a gradient descent step with respect to the combined objective in eq~\eqref{eq:ALMloss}. 
% 
% This method converges when the Lagrange multipliers in the augmented Lagrange function are close to the exact Lagrange multipliers of the optima. 
% 
% Unlike the quadratic penalty method, 
In practice, ALM can operate 
% quickly converges to exactly satisfy the constraints 
with moderate values of $\mu$ ($\sim 10$) without fine tuning, 
and thus exhibits fast and stable convergence 
% this method does not require the parameter $\mu$ to be 
% indefinitely large and as a result it does not suffer from training instabilities 
\citep{bertsekas1976multiplier}.

%% file: linear_setting.tex
\section{Analytic Results} %Linear Setting [Linear Problem?]} 
\label{sec:linear_setting}

\subsection{General Linear SEM}

In this section, we demonstrate that MRI-v1 can effectively eliminate all features that are spuriously correlated with the label in a linear predictor, given a sufficient number of environments. % $|\mathcal{E}_\text{tr}| \geq d_s$, where $d_s$ is the dimension of spurious features. 
% 
% {\red [outline: 
% Theorem 1: Given de+1 environments, MRI and MRIl can eliminate all spurious feature dimensions of a linear predictor while not constraining the invariant feature dimensions. True for both linear regression and classification.}
% 
Consider a data generating process according to Fig~\ref{fig:sem}, 
in which 
the observation $\inp = g(\latent_\causal,\latent_\spu)$ is an injective linear function
of the latent features $\latent_\causal \in \mathbb{R}^{d_\causal}, \latent_\spu \in \mathbb{R}^{d_\spu}$. 
% {\cyan and  $\expect_e[\latent_i|\labl]$ is a linear function of $\labl$}.
%   
% \begin{align}
%         \expect_e[\latent_\causal|\labl] &= \labl m_\causal , ~~~~
%         \expect_e[\latent_\spu|\labl] &=\labl m_e
%     \label{eq:latentexpectation}
% \end{align}
% 
Note that this  
{\blue Structural Equation Model (SEM) \citep{pearl2009causality} 
does not require any assumptions on the generation process $\Labl \to \Latent_\causal, \Latent_\spu$,
which generalizes the SEM of \cite{rosenfeld2020risks}, which additionally assumed binary labels and additive Gaussian noise for generating the latent features.}
% a specific, noise model ({\it i.e.} {\cyan linear} additive Gaussian) for the latent features.} 
% {\green In comparison, we don't put any assumptions on the noise generating process or the label.}
% {\purple [Too Vague. ELABORATE. Also mention that we assume general noise distribution].}
%  $\expect_e[\latent_i|\labl]$ is a linear function of $\labl$.
% {\red We only assume that the mapping $\inp = g(\latent_\causal,\latent_\spu)$ is a linear. }

We consider a linear predictor $\predictor:\Inp \to \Out$.
Since $g$ is injective and has an inverse over its range, without loss of generality, we can define $\predictor$ as a linear function directly over the latents as
\begin{equation}
    \out = \predictor(x; w) = w_\causal^\intercal \cdot \latent_\causal+ w_\spu^\intercal \cdot \latent_\spu
    \label{eq:linear_predictor}
\end{equation}
with parameters $w \equiv \{ w_\causal \in \mathbb{R}^{d_\causal}, w_\spu \in \mathbb{R}^{d_\spu} \}$.
% where $d_i$ is the dimension of invariant features, $d_s$ is the dimension of spurious features.
% 
% such that $m_\causal \in \mathbb{R}^{d_{\causal}}$ is independent of the environment, whereas $m_e \in \mathbb{R}^{d_{\spu}}$ varies with the environment. Note that the expectation of the latent features given $\labl$ does not need to be a linear function of $\labl$ for invariance under MRI and $\MRIl$ constraints. We choose it to be linear because we consider the class of linear predictors. }
% 

\begin{theorem}
 Given $|\mathcal{E}_\text{tr}| > d_s$ training environments, and that $\expect_e[\latent_\spu \labl]$ are in general linear positions,
% 
%  if the linear predictor satisfies the MRI and $\MRIl$ constraint then it holds that $w_\spu=0$ and $w_\causal$ span $\mathbb{R}^{1 \times d_\causal}$ assuming $m_e$'s are linearly independent. 
MRI-v1 will eliminate all spurious feature dimensions.
%  In other words, all spurious feature dimensions of the predictor are eliminated but the invariant feature dimensions remain unconstrained.
\label{thm:linearmodeltheorem}
\end{theorem}

% Assume $\mu_\causal(\labl) = m_\causal \labl$ and $\mu_e(y) = m_e \labl$ are linear functions of y and means $M_e$ are linearly independent.

\begin{proof}

% Applying 
MRI-v1's constraint eq~\eqref{eq:MRIv1_conservation}
% to eq~\eqref{eq:linear_predictor} 
yields %is % can be expressed as
\begin{equation*}
%  \forall \labl \in \Labl,
\forall e \in \mathcal{E}_\text{tr}, ~~~
\expect_e[\out \labl] = w_\causal \cdot \expect_e[\latent_\causal \labl] + w_\spu \cdot \expect_e[\latent_\spu \labl] = const,
%  ~~~~~ 
%  \label{eq:SEMMRIconstraint}
\end{equation*}
which can be expressed in a matrix form as
\begin{equation} 
w_\spu \cdot M' = 0,
\label{eq:matrix_eq}
\end{equation}
where $M' = M - \bar{M} \cdot 1$ with $M \equiv [ \expect_e[\latent_\spu \labl] ]_{e \in \mathcal{E}_\text{tr}}  \in \mathbb{R}^{d_s \times|\mathcal{E}_\text{tr}|}$
and $ \bar{M} = \frac{1}{|\mathcal{E}_\text{tr}|} \sum_{e \in \mathcal{E}_\text{tr}} \expect_e[\latent_\spu \labl] $.
% therefore, 
% $w_\spu \cdot \expect_e[\latent_\causal|\labl] = const$

Since $rank(M') = d_s$, eq~\eqref{eq:matrix_eq} is 
equivalent to %true if and only if 
$w_\spu=0$.

% since
% $\expect_e[\latent_\causal|\labl]$ is constant, this yields
% $w_\spu  \expect_e[\latent_\causal|\labl] = const$

% since
% $\expect_e[\latent_\causal|\labl] \neq const$,
% this yields
% $w_\spu =0$. 
% }

% {\cyan
% Define a $d_\spu \times E$ matrix
% $M \equiv [\expect_1[\latent_\causal|\labl], \cdots, \expect_E[\latent_\causal|\labl] $
% for each $\labl$. 
% Since, $w_\spu M = const$, 
% $w_\spu =0$.
% }

% \begin{equation}
%  (w_\causal m_\causal + w_\spu m_1)\labl = %(w_\causal m_\causal + w_\spu m_2)\labl=  
%  \cdots = (w_\causal m_\causal + w_\spu m_E)\labl
%  ~~~~~ \forall \labl \in \Labl
%  \label{eq:SEMMRIconstraint}
% \end{equation}

% and the $\MRIl$ (Eq~\eqref{eq:MRI_constraint_linear}) constraint gives us

% \begin{equation}
%  (w_\causal m_\causal + w_\spu m_1)\expect_\labl[\labl^2] = (w_\causal m_\causal + w_\spu m_2)\expect_\labl[\labl^2]=  ... = (w_\causal m_\causal + w_\spu m_E)\expect_\labl[\labl^2]
%  \label{eq:SEMMRIlconstraint}
% \end{equation}

% Both Eq~\eqref{eq:SEMMRIconstraint} and \eqref{eq:SEMMRIlconstraint} can be rewritten as
% \begin{equation}
%     w_\spu (m_e - m_{e+1}) = 0 \quad \forall e = 1, 2, ..., E-1
%     \label{eq:linearSEMconstraint}
% \end{equation}
% Since $m_e$'s are linearly independent and $E-1 \geq d_s$, Eq~\eqref{eq:linearSEMconstraint} can be satisfied if and only if
% \begin{equation}
%     w_\spu = 0
% \end{equation}

\end{proof}
% 
% Theorem~\ref{thm:linearmodeltheorem} stating 
% \cite{rosenfeld2020risks} derived 
A similar result was shown for $\IRMorig$, but under a more restricted setting, such as  
% can eliminate all spurious feature dimensions given $|\mathcal{E}_\text{tr}| > d_\spu$. 
% However, unlike our results for MRI-v1, their result is restricted to 
a specialized linear family of environments with binary labels and additive Gaussian noise \citep{rosenfeld2020risks}. In fact, IRM-v1 has been shown to fail 
% to eliminate spurious features 
in more general linear problems with non-Gaussian noise \citep{arjovsky2019invariant, kamath2021does}. 
% {\cyan The simplicity of our proof depends on the fact that we conserve E[o|y] instead of E[y|o] because of which \cite{rosenfeld2020risks} required additional assumptions.}

\captionsetup[subfigure]{position=top,textfont=normalfont,singlelinecheck=off,justification=raggedright}
\begin{figure}[t]
	\begin{center}
	% \subfloat[]{\includegraphics[width=.6\columnwidth]{figures/ToyLinearRegression_panel1.pdf}}\quad
	% \subfloat[]{\includegraphics[width=.6\columnwidth]{figures/ToyLinearRegression_panel2.pdf}}\quad
	% \subfloat[]{\includegraphics[width=.6\columnwidth]{figures/ToyLinearRegression_panel3.pdf}}
	\includegraphics[width=.8\columnwidth]{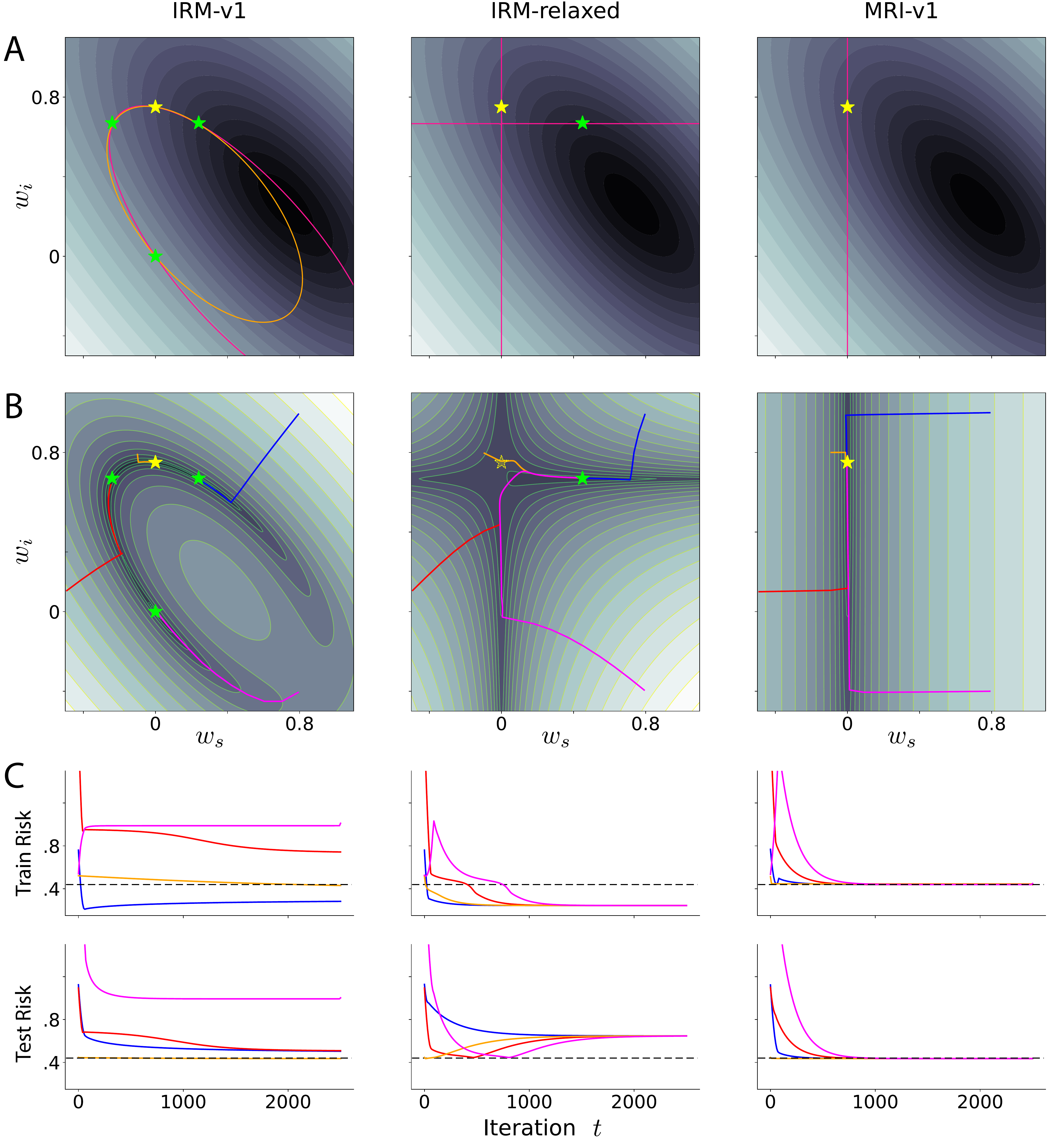}
	\caption{	Minimal example in Section~\ref{sec:minimal_example}.
	\blue 
	(A) Hard-constraint case.
	% Constraints overlaid with averaged risk over training environments. 
	Background: Average train domain risk $\loss_\text{tr}(f(w))$. 
	Lines: Invariance constraints $\vec{c}(f(w))=\vec{0}$.
	Stars: Local constrained optima.
	(B) Soft-constraint case.
	% Averaged risk with soft constraints. 
	Background: Train loss $\loss_\text{tr}(f(w)) + \mu \left\lVert \vec{c}(f(w)) \right\rVert ^2$.
	% added to $\loss_\text{tr}(\predictor(w))$ according to the penalty method 
	% eq~\eqref{eq:penalty_loss}.
% 	Exemplary 
	Convergence trajectories $w^t$ from multiple initialization are shown. %overlaid.
	%  with different colors denote different initialization. 
% 	Stars: same as panel a. Axes: same as panel a. The trajectories corresponding to different initialization are denoted by different colors. {\cyan Stars denote local minima of the objective.}
	% 
	(C) Loss profiles of the above convergence trajectories.
	Color-matched with panel B. 
	(Top) Average train domain risk $\loss_\text{tr}(f(w^t))$.
	(Bottom) Test domain risk $\loss_\text{test}(f(w^t))$.
	Dashed line: True optimal invariant solution (yellow star in panel A/B).
	}
	\label{fig:loss_landscape}
	\end{center}
	\vspace{-0.5cm}
\end{figure}

\subsection{Minimal Example: $d_i=1, ~d_s=1, ~ |\mathcal{E}_\text{tr}|=2$}
\label{sec:minimal_example}

{\blue
% \subsection{{\blue Example:} Shape-Texture Linear Regression}

% {\red[change the figure axis labels to $W_i$/$W_s$?]}
% [1. Analytical result: 
% {\red In this section, we demonstrate analytically in the linear shape-texture texture regression problem the }

% Here, we consider a {\cyan linear} regression problem on the {\blue linear version of the} Shape-texture dataset 
Here,
we demonstrate 
% the above result 
% for the general linear SEM above 
% with a minimal linear regression problem 
a minimal case of the above general linear problem
that involves one invariant feature, one spurious feature and two training environments 
% {\red which has 2 training environments and include 1 invariant and 1 spurious features [WHOSE input is a concatenation of the invariant and spurious features].}
% This example shares the same structure as 
(Shape-Texture linear regression problem in Section~\ref{sec:dataset}).
% {\cyan It is analogous to the two-bit environment task from \cite{kamath2021does}.}
}
See Supplementary Materials for the detailed experimental set up and the analytical solutions.
% 
% We show similar 
% The analysis results for 
A similar minimal examples for linear binary classification are also analyzed and shown in Supplementary Materials (Fig~\ref{fig:linear_classification}, linear shape-texture classification and toy-CMNISTa/b).

%~\ref\ref{sec:Appendix_linear_shape_texture}
% The results for the hard-constrained problem are shown in Fig~\ref{fig:loss_landscape}a.
\paragraph{Analytic solutions (Hard-constraints, Fig~\ref{fig:loss_landscape}A)}
% Shown in Fig~\ref{fig:loss_landscape}A. %shows the results for the 

% In this section, we derive the analytical result of the linear shape texture regression problem (see \ref{sec:Appendix_linear_shape_texture}) 

% \paragraph{IRM-v1}

IRM-v1 has two quadratic equality constraints,
{\blue 
% ($\delta\loss_{e_1}(\predictor) =  \delta\loss_{e_2}(\predictor)=0$),
% $\vec{\delta\loss} = \vec{0}$,
% {\cyan 
% where $ %\vec{c}
% \vec{\delta\loss} = [\delta\loss_{e_1},  \delta\loss_{e_2}]$
% and $\delta\loss_{e} =\expect_e [  \out^2 - \out \labl  ]$ 
% }
$\expect_{e_1} [  \out^2 - \out \labl  ] = \expect_{e_2} [  \out^2 - \out \labl  ] = 0$, 
}
shown as two elliptic curves. % in Fig~\ref{fig:loss_landscape}A. 
The intersection between the non-convex constraints on a 2-D feature space 
yields a disjoint set of 0-D points,
{\blue all of which are local constrained optima,
%  {\red shown as 4 stars.[REMOVE THIS. DESCRIBE IN LEGEND]} % each of which satisfies IRM-v1's solution. 
% Since this problem contains two degrees of freedom and two constraints, and due to the nonconvexity of constraints,  means that we would be left  with zero DOF ({\it i.e.}, points). 
% Also since quadratic equality constraints are non-convex, their intersection produces multiple disjoint sets (Fig~\ref{fig:loss_landscape}A).
% Since the solutions to the constraint are disjoint, each of them are indeed constrained optimum solution of the objective $\loss_\text{tr}(\predictor)$.
% 
% Note that among the 4 solutions, only two of them correspond to 
including the true invariant optimum, % (yellow star), % ($w_s=0$) predictor, 
a zero-predictor solution, and two non-invariant solutions.
Note that one of the non-invariant solutions exhibits lower train loss than the true invariant optimum solution.}
% 
% {\red This result confirms the inconsistency of IRM-v1 at identifying invariant predictors \citep{kamath2021does,arjovsky2019invariant}.}

% {\cyan (note that one of the non-invariant solutions has lower risk than the invariant one).}

% \paragraph{IRM-relaxed}

{\blue 
We also test the relaxed version of IRM-v1 
% effect of relaxing 
by removing the extraneous constraint.
% To isolate the non-convexity issue of the $\IRMorig$ constraint,  
% The extraneous constraint of 
% IRM-v1 can be relaxed to 
% yield 
% use just one constraint,
% $(\delta\loss_{e_1} - \delta\loss_{e_2})/\sqrt{2}=0$,
{\blue The relaxed version has a single constraint  
$ % Q\vec{\delta\loss} =  (\delta\loss_{e_1} - \delta\loss_{e_2})/\sqrt{2}
\expect_{e_1} [ \out^2 - \out \labl ] = \expect_{e_2} [ \out^2 - \out \labl ] $,
}
% {\cyan similar to MRI-v1.}
% This constraint features 
which describes a pair of hyperbolic curves (appears as two straight lines in Fig~\ref{fig:loss_landscape}A),}
% Since it has only a single constraint, it has 1d continuous set of points as possible solutions (Fig~\ref{fig:loss_landscape}A). 
% However, it is still a nonconvex constraint. 
% In addition to the $w_s=0$ line, it also exhibits {\red $w_i={2}/{3}$} line as the solution (Fig~\ref{fig:loss_landscape}A). 
% As a result, the $\loss_\text{tr}(\predictor)$ objective 
{\it i.e.} a non-convex constraint. This problem exhibits two local constrained optima,
including the true invariant optimum and a non-invariant solution,
with the non-invariant solution exhibiting a lower train loss. 
% This shows that % 
Therefore, 
simply relaxing the extraneous constraint
does not resolve the {\red fundamental problem} of IRM-v1. % {\red without fixing 
% % the fundamental problem: {\it i.e.} 
% the indirect way IRM imposes invariance}.
% the fundamental issue that IRM imposes invariance in an indirect way.}

% \paragraph{MRI-v1}
In contrast, 
MRI-v1 has one linear equality constraint, 
{\blue $ \expect_{e_1} [  \out \labl ] = \expect_{e_2} [  \out \labl ]  $,}
{\red that exactly prescribes the set of all invariant solutions $w_s=0$.}
{\blue
That is, a solution is an invariant predictor \emph{if and only if it satisfies this constraint.}
This is a convex problem, since both the objective and the constraint are  convex, and thus features a unique optimum, 
which is the true invariant optimum solution. 
}

% \subsubsection{Learning Dynamics under Numerical Optimization} %

\paragraph{Convergence Dynamics (Soft-constraints, Fig~\ref{fig:loss_landscape}B/C)}

Here, we analyze 
% the convergence trajectories of the
% soft-constrained 
the optimization dynamics %of the problem 
under full-batch gradient descent (penalty method  eq~\eqref{eq:penalty_loss} with $\mu = 5\times10^4$).
% 
% [The point of this paragraph is to show that for IRM-v1, the final result is highly dependent on the initialization, and that it can be extremely difficult/unlikely to converge to the invariant optimum.]
% The convergence trajectories of 
IRM-v1's convergence is highly dependent on the initialization (shown with differently colored trajectories), due to the presence of multiple local minima. 
% In this case, 
Note that most trajectories do not converge to the true invariant optimum. 
% as it has higher train loss than the nearby non-invariant local minima.
% 
IRM-relaxed also exhibits complex dynamics due to a saddle point near the true invariant solution: 
% , and a non-invariant optimum. 
% Therefore, 
Some trajectories (red and magenta) first approach the line of invariant solutions, 
but all trajectories eventually converge to the non-invariant solution. 
Note that the true invariant optimum solution is not even a local optimum of IRM-relaxed at this value of $\mu$, but it would exist in the limit $\mu\to\infty$. 
% As a result, all trajectories converge to the non-invariant optimum. % , which has lower training risk.
% 
In contrast, MRI-v1 always converges to the true invariant optimum regardless of initialization since it is a unique minimum.

%% file: non_linear_setting.tex
\section{Nonlinear Image-based Problems}

{\green 
Unlike for linear problems, theoretical proof for invariance is difficult to show for nonlinear problems.
Here, we empirically investigate  % we show the empirical investigation on
the performance of IRM-v1 and MRI-v1 in nonlinear image-based problems.}

% Classification:
% 1. For binary classification, we have analytical results guaranteeing the MRI will get rid of spurious dimensions in the predictor. However, IRM still suffers from the problem of multiple local minima due to non-convexity. Here we demonstrate the results for two datasets: Shape-Texture classification and CMNIST.

% Regression:
% 1. We don't have analytical results guaranteeing that MRI should work in this case, however, we believe that convexity of the MRI constraints provides an advantage over IRM as it does not have multiple disconnected minima in the parameter space. 
% 2. Empirically, we find that MRI finds a unique invariant minima (train and test loss matches) irrespective of initialization. However, IRM }

\subsection{Datasets}
\label{sec:dataset}

\paragraph{Shape-Texture Dataset}
% In order to analyze the OOD generalization performance of the algorithms in simple linear and nonlinear settings, 
% 
{\green We introduce a new dataset that is designed to evaluate domain generalization algorithms across various settings, 
% including both linear and nonlinear  regression and classification tasks.
including linear regression, linear classification, nonlinear image-based regression, and nonlinear image-based classification.
}
% 
% We introduce a new versatile dataset designed for domain generalization
% called the Shape-Texture dataset. 
% , based on the data generating model in Fig~\ref{fig:sem}.
% This dataset uses the same task structure as the Colored-MNIST task - the spurious feature correlates more strongly (as compared to the invariant feature) with the label in the training environments while the correlation between invariant feature and the label remain unchanged across training and testing environments. 
% This dataset is designed to test algorithms 
% {\red on linear and nonlinear image-based regression and classification problems.}
%in multiple adjustable settings.
% , which is described in detail in the methods section. 
% 
% The dataset involves 
% an invariant latent feature 
% % $\latent_\causal = 
% $\theta_\causal \in \mathbb{R}$, 
% a spurious latent feature %$\latent_\spu = 
% $\theta_\spu \in \mathbb{R}$, 
% and  the latent feature $\theta_\labl \in \mathbb{R}$ for the label.
{\blue 
The generative process of the dataset involves 
an invariant feature $\latent_\causal = e^{i\theta_\causal}$,
a spurious feature $\latent_\spu = e^{i\theta_\spu}$,
and a label feature $ e^{i\theta_\labl}$,
% which are related according to {\red SEM (Fig~\ref{fig:sem})}.  
each of which 
represents an orientation on a complex unit circle: $e^{i\theta} \in \mathcal{S}^1$. 
% and therefore is naturally represented on a unit circle in the complex plane as $e^{i\theta} \in \mathcal{S}^1$.%{\cyan = \cos(\theta) + i\sin(\theta)} 
% 
The angles of orientations are generated as %related as
$\theta_\labl \sim \mathcal{U}_{S^1}$, % any interval $\Gamma$ of length $2\pi$
where $\mathcal{U}_{S^1}$ is the circular uniform distribution, and 
for $* \in \{\causal,\spu\}$, $\theta_* = \theta_\labl$ with probability $p_*$ 
or $\theta_* \sim  \mathcal{U}_{S^1}$ with probability $1-p_*$.
% $\theta_* \leftarrow \theta_\labl + \xi_*$  where  $\xi_* = 0$ with probability $p_*$   % \sim \mathcal{N}(0,\sigma_*)$ or $\xi_* \sim  \mathcal{U}_{S^1}$ with probability $1-p_*$.
% 
The parameter $p_\causal = 0.75$  % $\sigma_\causal=0$ are  
is fixed across environments, % invariant parameters
whereas  $p_\spu$ % $\sigma_\spu$ 
varies from one environment to another. 
We consider two training environments $\mathcal{E}_\text{tr} = \{ e_1, e_2 \}$ with 
$p_{\spu_{e_1}} = 1, ~ p_{\spu_{e_1}} = 0.8$ and one testing environment $\mathcal{E}_\text{test} = \{ e_0 \}$ with $p_{\spu_{e_0}} = 0$. 
}

In the linear regression task (section~\ref{sec:minimal_example}), the observed input is the concatenated latent features %{\red scalar complex numbers}
$\inp = [e^{i\theta_\causal}, e^{i\theta_\spu}]$
and the label is $\labl = e^{i\theta_\labl}$. 
In the linear classification task, the input is $\inp = [\sin(\theta_\causal), \sin(\theta_\spu)]$ and the label is $\labl = H(\sin(\theta_\labl))$, where $H$ is {\blue the sign function}. 
In the nonlinear regression/classification tasks, 
% the label remains the same %as their respective linear versions but 
the observed input is the image composed of two planar waves,
in which $\theta_\causal$ is the orientation of the low frequency wave ({\it i.e.} shape) 
and $\theta_\spu$ is that of the high frequency wave ({\it i.e.} texture),
% The orientation of the low frequency wave (a.k.a shape) is given by $\theta_\causal/2$, and the orientation of the high frequency wave (a.k.a texture) is given by $\theta_\spu/2$
as shown in Fig~\ref{fig:images_shape_texture}.
% 
% the shape feature is the causal features corresponding to the low frequency feature in the input image and it's orientation is given by 
% 
% {\cyan In both cases, the label is $\labl = e^{i\theta_\labl}$}
% The label for both linear and nonlinear settings is given by a scalar complex number $e^{i\theta_\labl}$. 
% 
% 
% For both linear and nonlinear settings, 
% The relationship between $\theta_\labl$ and $\theta_\causal$ remains invariant across all environments but the relationship between $\theta_\labl$ and $\theta_\spu$ varies from one environment to another.
% 

\paragraph{Colored MNIST (CMNIST)}
CMNIST \citep{arjovsky2019invariant} is a synthetic dataset derived from MNIST for binary classification.
% to analyze $\IRMorig$'s ability to learn nonlinear invariant predictors. 
% In this dataset,
{\green
In this dataset, the label $\labl$ assigned to an image is based on the digit bit $\latent_\causal$ 
(1 for digits $0\sim4$ and -1 for $5\sim9$) 
such that $y=\latent_\causal$ with  probability $p_\causal$ or $-\latent_\causal$ with  probability $1-p_\causal$. 
The color bit $\latent_\spu$ (1 for red -1 for green) is chosen based on the label such that $\latent_\spu=\labl$ with probability $p_\spu$ or $-\labl$ with probability $1-p_\spu$.
% with $p_\spu$ changing across environments. 
% The image is colored red if $\latent_\spu=1$ or green if $\latent_\spu=-1$.
}
% The dataset features a fixed flipping probability $p_\causal$ between the digit and the label, % is fixed, %remains the same 
% % across all environments,
% and $p_\spu$ between the color and the label,
% % The color of the image is determined in correlation with the label whose correlation probability $p_\spu$ 
% which changes across environments. 
% Therefore,  color is the spurious feature and digit is the invariant feature. This dataset follows the data generating model in Fig~\ref{fig:sem}. 
% 
% In this work, 
We consider two versions, CMNISTa and CMNISTb, with two sets of environmental parameters: % $p_\spu$ and $p_\causal$. 
% The first version, 
CMNISTa is the version from  \citet{arjovsky2019invariant} 
with $p_\causal = 0.75$, 
$p_{\spu_{e_1}} = 0.9, ~ p_{\spu_{e_1}} = 0.8$, and $p_{\spu_{e_0}} = 0.1$ 
for the training $\mathcal{E}_\text{tr} = \{ e_1, e_2 \}$ and the testing  $\mathcal{E}_\text{test} = \{ e_0 \}$ environments;
% For the other version, 
CMNISTb uses 
$p_\causal = 0.9$, 
$p_{\spu_{e_1}} = 1, ~ p_{\spu_{e_1}} = 0.8$, $p_{\spu_{e_0}} = 0.1$.
% $p_\causal = 0.95$, $p_\spu = 0.8, 1$ for the two training environments, and $p_\spu=0.1$ for the testing environment. 
% 
{\blue In the nonlinear tasks, the input observation $x$ is the colored MNIST image.
In the abstracted versions, called toy-CMNISTa/b, the input observation is the two-bit  data $x = [\latent_\causal, \latent_\spu]$
% , where $\latent_\causal, \latent_\spu \in \{-1, 1\}$ correspond to the digit bit and the color bit, respectively 
\citep{kamath2021does}.
}

{\blue
\paragraph{Remark}
Note that both Shape-Texture and CMNIST datasets can be equally understood as being generated from Fig~\ref{fig:sem}
with causal directions $\Labl \to \Latent_\causal$ or $\Latent_\causal \to \Labl$ (Fig~\ref{fig:joint_dist}).
}

\subsection{Result} 

\begin{table}[t]
    \begin{subtable}[c]{1\textwidth}
    \centering
    \vspace*{1mm}
    \caption*{\textbf{Risk}}
    \scalebox{0.68}{
    \begin{tabular}{M{2.5cm}|cc|cc|cc|cc} 
    \toprule
     \multirow{2}{*}{}
     & \multicolumn{2}{c|}{S-T Regression} & \multicolumn{2}{c|}{S-T Classification} &  \multicolumn{2}{c|}{CMNISTa} &
     \multicolumn{2}{c}{CMNISTb}\\
     &  $\text{Train}$ & $\text{Test}$ & $\text{Train}$ & $\text{Test}$  & $\text{Train}$ & $\text{Test}$ & $\text{Train}$ & $\text{Test}$ \\ 
    \midrule
    \grey Oracle
    & \grey $0.46 \pm 0.00$ & \grey $0.46 \pm 0.00$ 
    & \grey $0.47 \pm 0.00$ & \grey $0.47 \pm 0.00$ 
    & \grey $0.57 \pm 0.00$ & \grey $0.58 \pm 0.00$
    & \grey $0.22 \pm 0.00$ & \grey $0.24 \pm 0.01 $ \\[0.3cm]
    ERM 
    & $0.16 \pm 0.00$ & $1.37 \pm 0.01$ 
    & $0.24 \pm 0.00$ & $1.16 \pm 0.01 $ 
    & $0.36 \pm 0.00$ & $1.44 \pm 0.01 $ 
    & $0.13 \pm 0.00$ & $0.73 \pm 0.01 $  \\[0.3cm]
    IRM-v1 (PM)
    & $1.00 \pm 0.00$ & $1.00 \pm 0.00$ 
    & $0.69 \pm 0.00$ & $0.69 \pm 0.00$ 
    & $0.69 \pm 0.00$ & $0.69 \pm 0.00$ 
    & $0.69 \pm 0.00$ & $0.69 \pm 0.00$ \\[0.3cm]
    IRM-v1 (ALM)
    & $0.23 \pm 0.00$ & $0.62 \pm 0.01 $ 
    & $0.4 \pm 0.01 $ & $0.51 \pm 0.01 $
    & $0.62 \pm 0.02 $ & $0.69 \pm 0.02 $ 
    & $0.17 \pm 0.01 $ & $0.47 \pm 0.02 $  \\[0.3cm]
    MRI-v1 (PM)
    & $0.53 \pm 0.03 $ & $0.54 \pm 0.03 $ 
    & $0.44 \pm 0.01 $ & $\bm{0.46} \pm 0.01 $ 
    & $0.62 \pm 0.01 $ & $0.66 \pm 0.01 $ 
    & $0.46 \pm 0.01 $ & $0.4 \pm 0.01 $
    \\[0.3cm]
    MRI-v1 (ALM)
    & $0.45 \pm 0.01 $ & $\bm{0.46} \pm 0.00$
    & $0.47 \pm 0.00$ & $ \bm{0.47} \pm 0.00$
    & $0.63 \pm 0.02 $ & $\bm{0.64} \pm 0.01 $
    & $0.25 \pm 0.01 $ & $\bm{0.29} \pm 0.01 $ \\[0.05cm] 
    \bottomrule
    \end{tabular}}
    % \label{tab:main_results}
    % \vspace{0.1cm}
    % \subcaption{Risk}
    \end{subtable}
    \begin{subtable}[c]{1\textwidth}
    \centering
    \vspace*{1mm}
    \caption*{\textbf{Accuracy}}
    \scalebox{0.7}{
    \begin{tabular}{M{2.5cm}|cc|cc|cc} \toprule
     \multirow{2}{*}{}
    %  & \multicolumn{2}{c|}{\phantom {S-T Regression}}
     & \multicolumn{2}{c|}{S-T Classification} 
     & \multicolumn{2}{c|}{CMNISTa} 
     & \multicolumn{2}{c}{CMNISTb}
     \\
    % &   & 
    &  $\text{Train}$ & $\text{Test}$  
    & $\text{Train}$ & $\text{Test}$
    & $\text{Train}$ & $\text{Test}$ 
    \\ 
    \midrule
     \grey Oracle
    % &  &
    &  \grey $0.86 \pm 0.00 $ &  \grey $0.86 \pm 0.00 $
    &  \grey $0.74 \pm 0.00 $ &  \grey $0.74 \pm 0.00 $
    &  \grey $0.94 \pm 0.00 $ &  \grey $0.94 \pm 0.00 $ \\[0.3cm]
    ERM 
    % &  &
    & $0.95 \pm 0.00 $ & $0.50 \pm 0.00 $ 
    & $0.85 \pm 0.00 $ & $0.10 \pm 0.00 $ 
    & $0.94 \pm 0.00 $ & $0.80 \pm 0.02 $ \\[0.3cm]
    IRM-v1 (PM)
    % &  &
    & $0.52 \pm 0.02 $ & $0.52 \pm 0.04 $ 
    & $0.73 \pm 0.14 $ & $0.23 \pm 0.16 $ 
    & $0.6 \pm 0.03 $ & $0.52 \pm 0.05 $ \\[0.3cm]
    IRM-v1 (ALM)
    % &  &
    & $0.79 \pm 0.01 $ & $0.77 \pm 0.01 $ 
    & $0.64 \pm 0.03 $ & $\bm{0.66} \pm 0.03 $
    & $0.93 \pm 0.00 $ & $\bm{0.91} \pm 0.00 $ \\[0.3cm]
    MRI-v1 (PM)
    % &  &
    & $0.86 \pm 0.01 $ & $\bm{0.85} \pm 0.01 $ 
    & $0.68 \pm 0.01 $ & $0.63 \pm 0.02 $ 
    & $0.82 \pm 0.01 $ & $0.86 \pm 0.01 $
    \\[0.3cm]
    MRI-v1 (ALM)
    % & \phantom{$0.00 \pm 0.00 $} & \phantom{$0.00 \pm 0.00 $}
    & $0.86 \pm 0.01 $ & $\bm{0.86} \pm 0.01 $
    & $0.66 \pm 0.02 $ & $\bm{0.65} \pm 0.02 $
    & $0.93 \pm 0.00 $ & $\bm{0.9} \pm 0.00 $  \\[0.05cm] 
    \bottomrule
    \end{tabular}}
    % \label{tab:main_results}
    \vspace{0.1cm}
    % \subcaption{Accuracy}
    \end{subtable}
    \caption{
    %Loss comparison 
    Comparison of algorithms on the Shape-Texture (S-T) and the Colored MNIST-a/b datasets:
    (Top) average risk $\loss_\text{tr}$, $\loss_\text{test}$, and 
    (bottom) accuracy. Oracle uses environments in which the spurious features are uncorrelated with the label. Mean and standard deviation shown up to 2 decimal places.}
    % \caption{Average train domain and test domain accuracy comparison achieved by different methods in image-based tasks. Oracle uses environments in which the spurious features are uncorrelated with the label. Mean and standard deviation shown up to 2 decimal places.}
    \label{tab:main_results}
\end{table}

Here, %In Table~\ref{tab:main_results}, 
% In addition to  %reporting the results for 
we report the performance of $\IRMorig$, $\MRIl$, 
{\blue as well as the vanilla Empirical Risk Minimization (ERM) algorithm ({\it i.e.} without imposing any invariance constraint).}
% algorithm, which simply minimizes the average empirical risk without any invariance constraint. 
% For synthetic datasets, 
For reference, 
the results are compared to the {\it Oracle}  performance,
which is obtained by
% we also report 
applying ERM on the modified training datasets in which the spurious features are rendered uncorrelated with the label. % called the {\it Oracle} performance.

We tested the algorithms under a wide range of hyperparameters for both the Shape-Texture  (Fig~\ref{fig:ST_PM}) and the CMNIST-b (Fig~\ref{fig:CMNIST_PM}) datasets. 
Overall, %In both PM and ALM settings, 
MRI-v1 consistently achieves good invariant performance close to the Oracle, 
whereas IRM-v1 often shows either chance-level performance or poor OOD generalization with large differences between the train and the test domain risks. 
% 
% with significantly better test-domain performance than IRM-v1 
% in terms of loss. 
% This is also largely true for accuracy, except that for the ALM setting of CMNIST-b, IRM-v1 shows slightly higher accuracy than MRI-v1. 
% 
The results for a specific hyperparameter setting is shown in 
Table~\ref{tab:main_results}. 
% compares the performance of algorithms for a specific hyperparameter setting. 

Under the PM setting, we observe that IRM-v1 often drives the models to the {\it zero-predictor} solution,
{\it i.e.} making zero output regardless of the input, even with annealing the penalty term to be applied only in the later phase of training. This explains
% converges to 
% drives the model train models to zero-predictor 
% is observed to mostly
IRM-v1's identically low performance across train and test domains, 
consistent to the previously reported results in \cite{gulrajani2020search}. 
% rosenfeld2020risks,kamath2021does,}.
% 
In contrast,  MRI-v1 never drives the models to the zero-predictor solution, consistent with the finding in Sec~\ref{sec:minimal_example} that MRI-v1 does not have an local minima at the zero-predictor in the linear problem setting. 

Interestingly, the ALM setting greatly improves IRM-v1's performance especially  in terms of accuracy, while % for the CMNIST dataset. 
MRI-v1's performance remains relatively unchanged between PM and ALM. 
% As a result, 
Under the ALM setting,
IRM-v1 shows comparable or slightly higher accuracy than MRI-v1 for the CMNIST dataset.

In the Supplementary Materials, we also report the results for other recent domain generalization methods including MMD \citep{li2018domain}, GroupDRO \citep{sagawa2019distributionally}, and IB-IRM \citep{ahuja2021invariance}  on the Shape-Texture classification task,
which exhibit significantly worse performance than MRI-v1
% over a wide range range of hyperparameters 
(Fig~\ref{fig:ST_other}).
% {\red which demonstrates the superiority of MRI-v1 over other domain generalization methods as well.} 
% Additionally, we demonstrate that MRI-v1 also outperforms (Fig~\ref{fig:ST_other}) other recent domain generalization methods like MMD \citep{li2018domain}, GroupDRO \citep{sagawa2019distributionally}, and IB-IRM \citep{ahuja2021invariance}. 
% }

% \captionsetup[subfigure]{position=top,textfont=normalfont,singlelinecheck=off,justification=raggedright}
\begin{figure}%[t]
	\begin{center}
	% \subfloat[Risk]
 \includegraphics[width=0.9\columnwidth]{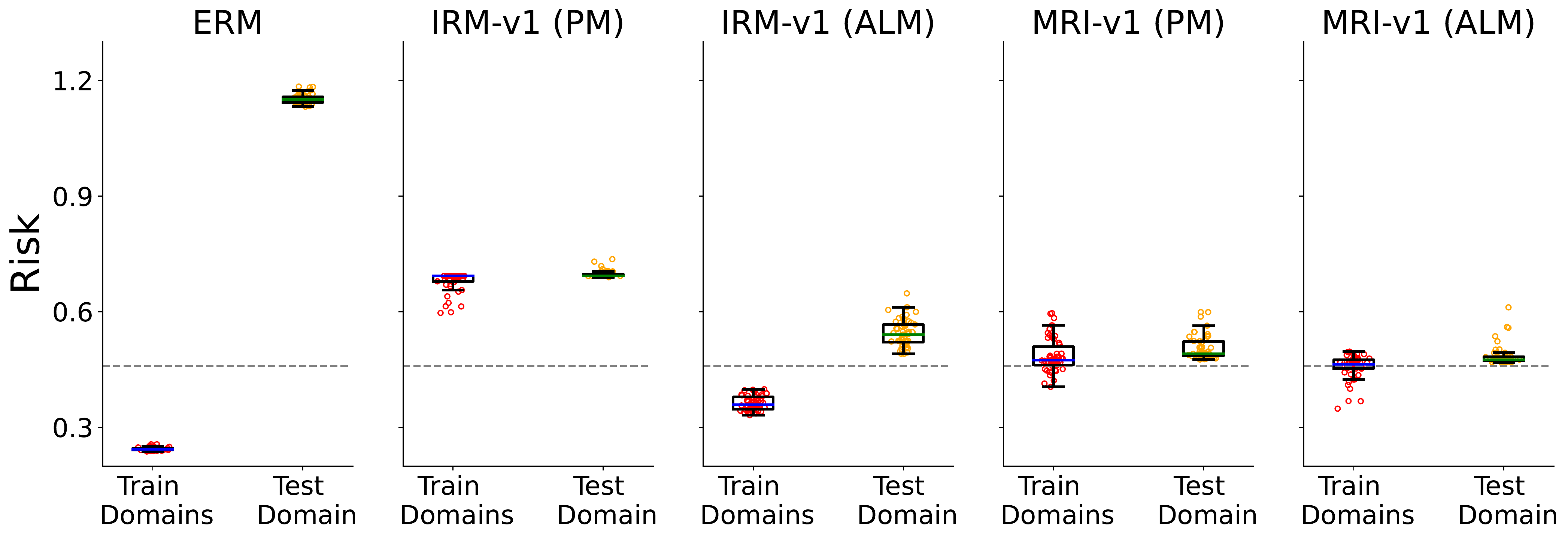}\quad  
	% \subfloat[Accuracy]
 \includegraphics[width=0.9\columnwidth]{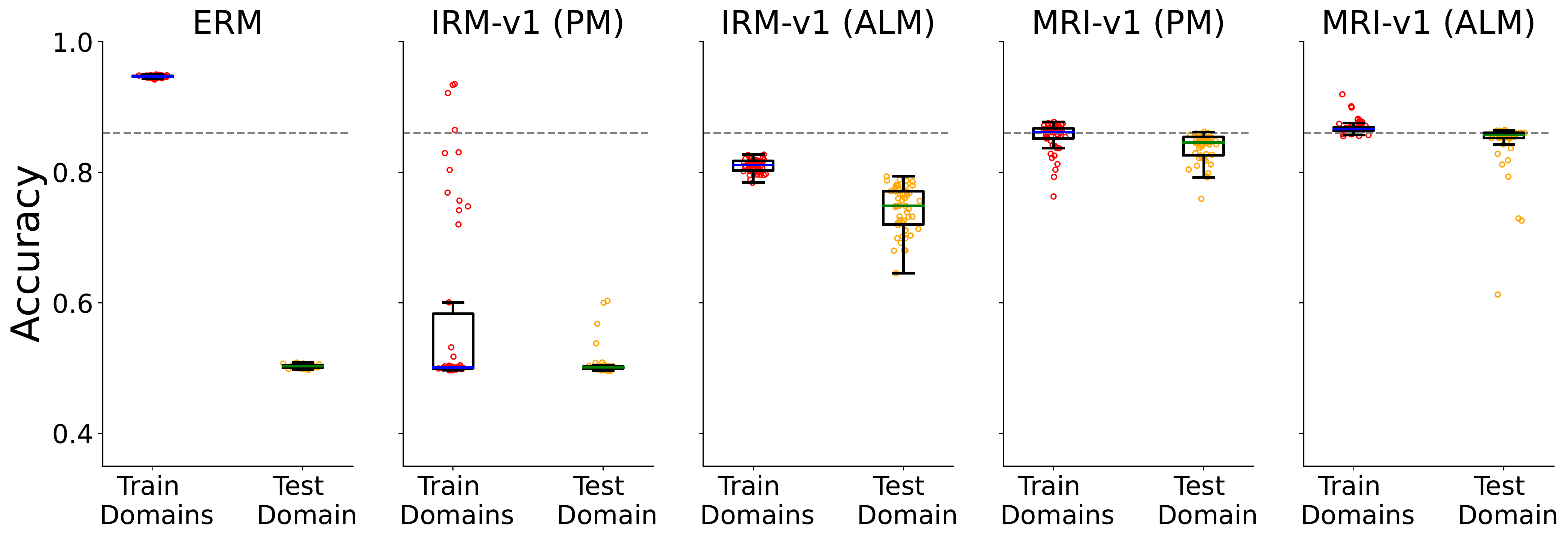}\quad
	\caption{
	Comparison of algorithms' performance over a range of hyperparameters 
	on the Shape-Texture classification dataset. 
	(Top) averaged train and test domain risk 
	and (bottom) accuracy.
% 	using 50 random hyperparameters for the following algorithms: ERM and MRI-v1 and IRM-v1 trained using both PM and ALM. 
	The grey dashed line denotes the Oracle performance. Box-plots show sample quartiles.
}
	\label{fig:ST_PM}
	\end{center}
\end{figure}
% 
% 
% \captionsetup[subfigure]{position=top,textfont=normalfont,singlelinecheck=off,justification=raggedright}
\begin{figure}%[ht]
	\begin{center}
	% \subfloat[Risk]
    \includegraphics[width=0.9\columnwidth]{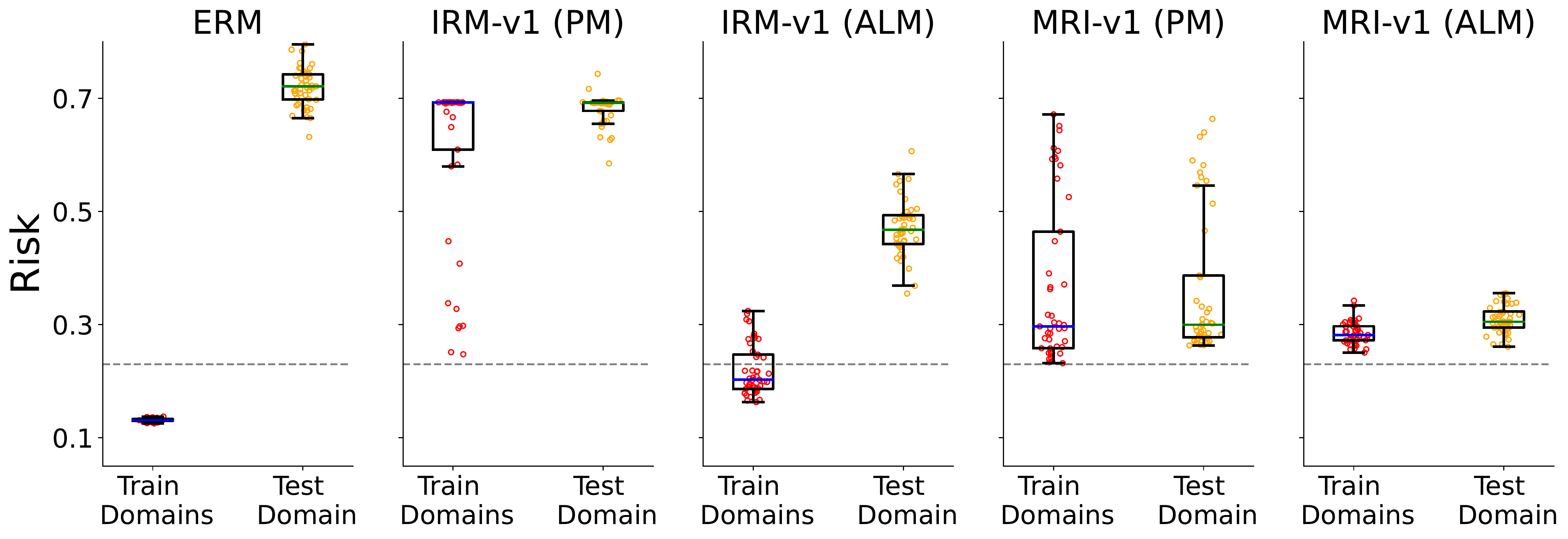}\quad
	% \subfloat[Accuracy]
    \includegraphics[width=0.9\columnwidth]{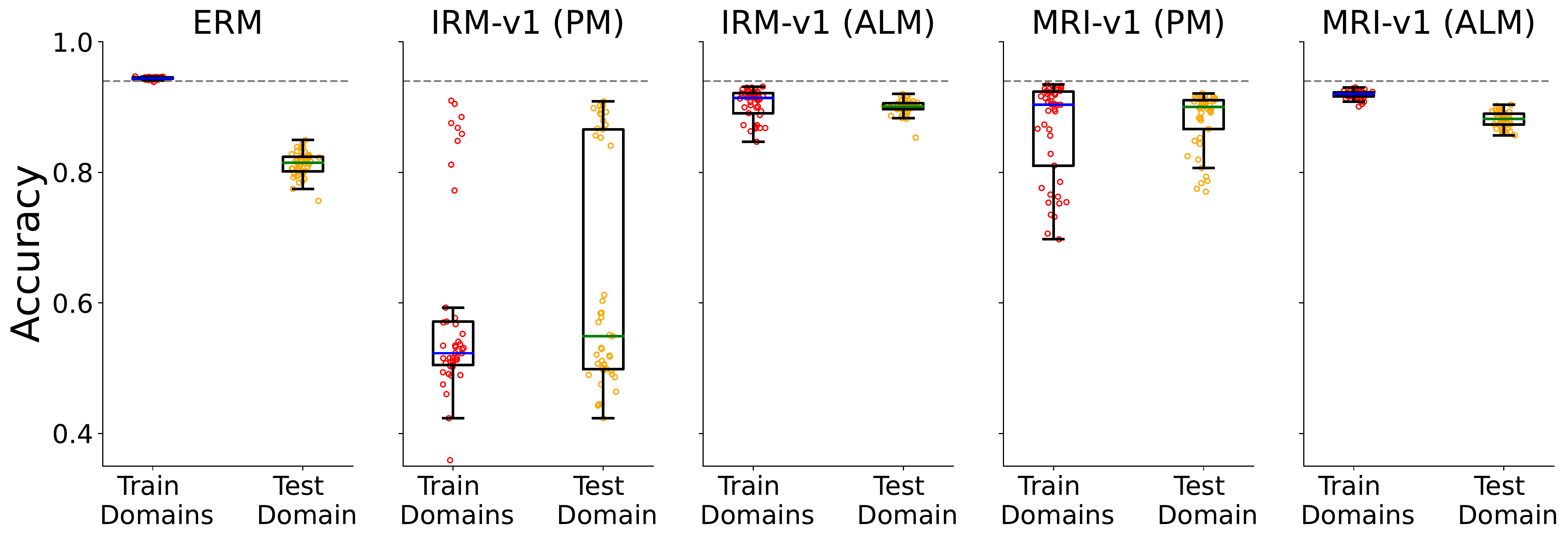}\quad
	\caption{
Result on CMNIST-b dataset. Same as Fig~\ref{fig:ST_PM}.
% 	(a)	Averaged train and test domain risk 
% 	and (b) accuracy
% 	using 50 random hyperparameters for IRM-v1 and MRI-v1 trained using both PM and ALM. Black
% dotted line denotes the Oracle performance.
}
	\label{fig:CMNIST_PM}
	\end{center}
\end{figure}

%% file: discussion.tex
\section{Discussion}
\label{sec:discussion}

{\green
\paragraph{Limitations} 

In principle, IRM 
% since 
only requires $\mathbb{E}_e[\labl|\latent_{\causal}]$ to be constant across domains in order to guarantee invariance,
and therefore it has been previously thought to be generally applicable in a wide range of problems,
even though this guarantee %is quite fragile and 
was only shown in the impractical limit of %when
optimizing over the unrestricted function space of $\modifier$. %$\delta\modifier$. % (or $\modifier$).
Here, we showed {\red a strong negative result} that 
% this invariance guarantee is in fact quite fragile
no {\cyan meaningful} form of invariance can be stated for IRM % guaranteed
when the function space of $\modifier$ is restricted.
% if any restriction is imposed on the function space of $\modifier$.
% 
% A major limitation of our proposed method is the assumption of no label shift across domains. \cite{mahajan2021domain} show that class-conditional invariance can fail if $P(z_c|y)$ does not remain the same across domains, which also applies to our method. Like many other invariant representation learning algorithms, including IRM \citep{ahuja2020empirical, ahuja2021invariance}, our method will be unable to handle covariate shift across domains. Also, like IRM \citep{ahuja2021invariance}, our method requires that there is significant overlap in the support of different environment distributions. 
% 
In contrast, MRI  % can guarantee invariance in problems 
requires a more limiting condition that
the label distribution $P(\labl)$ and 
$P(\latent_{\causal})$  % the invariant feature distribution 
should be constant across domains,
% {\red which is more limiting than IRM's requirement,}
% a narrower assumption,
{\red but 
% it yields invariance 
it can be more generally applied 
% in practice 
% yields practical applicability
% is practically applicable
% since it can be applied %its invariance guarantee is preserved 
even for the case of restricted function space of $\modifier$.
% when the function space of $\modifier$ is restricted,
which yields more practicality. % applicability.
}
% , which can behave very differently from IRM \citep{kamath2021does}, has been shown to fail for such problems \citep{ahuja2020empirical, ahuja2021invariance}. 

A common limitation of both MRI and IRM \citep{rosenfeld2020risks, ahuja2021invariance} is that they require significant support overlap across domains in order to guarantee OOD generalization. This may limit the applicability of these methods on certain domain generalization benchmarks % \citep{gulrajani2020search} 
that consist of domains that lack such overlap,
% including 
such as 
% most DomainBed \citep{gulrajani2020search} benchmarks. 
 % with different styles or natural backgrounds, like 
% {\it e.g.} 
VLCS (different image stylization, \cite{fang2013unbiased}), 
and Terra-Incognita (different natural backgrounds, \cite{beery2018recognition}). In the Supplementary Materials, we report that both methods do not show significant improvement over ERM on these datasets % (in terms of test accuracy) both VLCS and Terra-Incognita datasets 
(Table~\ref{tab:vlcs_acc_results}). 
%  {\red However, note that MRI-v1 does achieve a significantly better test loss than ERM on both the datasets (Table~\ref{tab:vlcs_loss_results}), which may suggest that MRI-v1 does achieve a higher degree of invariance than ERM.}
% 
}

{\green
\paragraph{Insensitivity of accuracy metric in evaluating invariance} 
% The accuracy metric is widely used in the field of domain generalization \citep{gulrajani2020search} to compare the performance of various algorithms. However, accuracy is not sensitive enough to reliably measure the degree of invariance of a model. 

% (The reviewers pointed out that 
% In the ALM setting, 
% In Table~\ref{tab:main_results},
For CMINST{\cyan a/}b, IRM-v1 (ALM) exhibits  comparable performance %that is comparable 
to MRI-v1  in terms of test domain accuracy,
despite having significantly worse performance in term of test domain risk
% The test domain accuracy of IRM-v1 and MRI-v1 may exhibit 
% {\red minor differences}, despite MRI-v1 outperforming IRM-v1 in terms of test domain risk 
(See Table~\ref{tab:main_results}). %, CMINSTa/b). 
We investigated this phenomenon %We provide an explanation for this phenomenon 
by analyzing the linear version of the task, 
% the accuracy landscape of the linear classification problem of 
toy-CMINSTa/b. % (Fig~\ref{fig:toy_CMNIST}). 
% While MRI-v1 has a unique optimum, which is the optimal invariant solution, IRM-v1 exhibits multiple local optima, including non-invariant solutions. Yet, all these solutions exhibit identical accuracy (except for the zero-predictor solution of IRM-v1), since the accuracy landscape (Fig~\ref{fig:toy_CMNIST}f) is flat over a wide range of solutions:
% {\red and remains invariant across domains:} 
The accuracy landscape %of toy-CMINSTa/b has piece-wise constant landscape, 
is piece-wise constant
% which is flat over a wide range of solutions 
(Fig~\ref{fig:toy_CMNIST}C,F). Especially, it exhibits identical ({\it i.e.} invariant) accuracy between train and test domain in the region defined by $w_i>|w_s|$. 
Therefore, the accuracy metric cannot distinguish
invariant solutions ($w_s=0$) from non-invariant solutions within the region.
% as long as $w_i>|w_s|$ is satisfied. 
% optimum of IRM-v1
In contrast, the train domain and the test domain risk share the same value only if the solution is invariant ($w_s=0$) (Fig~\ref{fig:toy_CMNIST}B,E).
{\red The constraint function of IRM-v1 (Fig~\ref{fig:toy_CMNIST}A,D) shows 
that toy-CMNISTa only has invariant local optima and 
that toy-CMNISTb has additional non-invariant local optima, 
% in additional to the invariant local optima 
all of which satisfy $w_i>|w_s|$, thus exhibiting the same accuracy performance as the invariant optimum solution.  
}
% both the non-invariant local optima of IRM-v1 all satisfy $w_i>|w_s|$, which explains }
% indeed fall into this region.} 
% all of the 
% which is satisfied not only by the invariant optimum of IRM-v1 ($w_s=0$), but also by the non-invariant optima. 
% Nevertheless, these solutions exhibit different test domain risk. 
This result illustrates that using accuracy metric alone 
for evaluating the degree of invariance could be insufficient,
% the insufficiency of using the accuracy metric alone 
% could be insufficient, 
% metric may not yield a sufficient indicator 
% is not a reliable indicator of the degree of invariance, % achieved by IRM-v1 and MRI-v1 
{\red and highlights the need to also consider risk for evaluations.}
% both test domain accuracy and 

% As a result, accuracy cannot be used as a measure of invariance in this problem. Therefore, the fact that MRI-v1 and IRM-v1 exhibit similar test-domain accuracy does not necessarily indicate that they achieve similar degree of invariance.

% However, accuracy is not sensitive enough to reliably measure the degree of invariance of a model. We demonstrate this in a simple toy problem, the toy-CMNISTb dataset. In this case, we find that the accuracy landscape is flat for a wide range of parameters in the weight space. As a consequence, the invariant optimum of MRI-v1 and the non-invariant optima of IRM-v1 achieve identical accuracy (Fig~\ref{fig:toy_CMNIST}). In comparison, the loss metric (Fig~\ref{fig:toy_CMNIST}) can indeed differentiate between the invariant and non-invariant representations.
% This discrepancy between the loss and accuracy metric holds true even for nonlinear image-based datasets (Table~\ref{tab:loss_results}-\ref{tab:vlcs_acc_results}). 
% To our knowledge, such a discrepancy between loss and accuracy has not been reported before in the domain generalization literature.
%
% Based on this, we suggest that the reliance on accuracy as the sole metric of comparison between domain generalization algorithms can be misleading, but loss and accuracy taken together can provide a meaningful comparison.
}

%% file: checklist.tex
\section*{Checklist}

%%% BEGIN INSTRUCTIONS %%%
The checklist follows the references.  Please
read the checklist guidelines carefully for information on how to answer these
questions.  For each question, change the default \answerTODO{} to \answerYes{},
\answerNo{}, or \answerNA{}.  You are strongly encouraged to include a {\bf
justification to your answer}, either by referencing the appropriate section of
your paper or providing a brief inline description.  For example:
\begin{itemize}
  \item Did you include the license to the code and datasets? \answerYes{See Section~??.}
  \item Did you include the license to the code and datasets? \answerNo{The code and the data are proprietary.}
  \item Did you include the license to the code and datasets? \answerNA{}
\end{itemize}
Please do not modify the questions and only use the provided macros for your
answers.  Note that the Checklist section does not count towards the page
limit.  In your paper, please delete this instructions block and only keep the
Checklist section heading above along with the questions/answers below.
%%% END INSTRUCTIONS %%%

\begin{enumerate}

\item For all authors...
\begin{enumerate}
  \item Do the main claims made in the abstract and introduction accurately reflect the paper's contributions and scope?
    \answerYes{}
  \item Did you describe the limitations of your work?
    \answerYes{} See Section~\ref{sec:discussion}.
  \item Did you discuss any potential negative societal impacts of your work?
    \answerNA{}
  \item Have you read the ethics review guidelines and ensured that your paper conforms to them?
    \answerYes{}
\end{enumerate}

\item If you are including theoretical results...
\begin{enumerate}
  \item Did you state the full set of assumptions of all theoretical results?
    \answerYes{} See Section~\ref{sec:IRMvsMRI} and \ref{sec:linear_setting}.
        \item Did you include complete proofs of all theoretical results?
    \answerYes{} See Section~\ref{sec:IRMvsMRI} and \ref{sec:linear_setting}.
\end{enumerate}

\item If you ran      experiments...
\begin{enumerate}
  \item Did you include the code, data, and instructions needed to reproduce the main experimental results (either in the supplemental material or as a URL)?
    \answerYes{} \url{https://github.com/IBM/MRI}
  \item Did you specify all the training details (e.g., data splits, hyperparameters, how they were chosen)?
    \answerYes{} See Supplementary Materials.
        \item Did you report error bars (e.g., with respect to the random seed after running experiments multiple times)?
    \answerYes{} See Table~\ref{tab:main_results}.
        \item Did you include the total amount of compute and the type of resources used (e.g., type of GPUs, internal cluster, or cloud provider)?
    \answerYes{} See Supplementary Materials.
\end{enumerate}

\item If you are using existing assets (e.g., code, data, models) or curating/releasing new assets...
\begin{enumerate}
  \item If your work uses existing assets, did you cite the creators?
    \answerYes{} See Section~\ref{sec:dataset} and Supplementary Materials.
  \item Did you mention the license of the assets?
    \answerYes{} See Supplementary Materials.
  \item Did you include any new assets either in the supplemental material or as a URL?
    \answerYes{} \url{https://github.com/IBM/MRI}
  \item Did you discuss whether and how consent was obtained from people whose data you're using/curating?
    \answerNA{}
  \item Did you discuss whether the data you are using/curating contains personally identifiable information or offensive content?
    \answerNA{}
\end{enumerate}

\item If you used crowdsourcing or conducted research with human subjects...
\begin{enumerate}
  \item Did you include the full text of instructions given to participants and screenshots, if applicable?
    \answerNA{}
  \item Did you describe any potential participant risks, with links to Institutional Review Board (IRB) approvals, if applicable?
    \answerNA{}
  \item Did you include the estimated hourly wage paid to participants and the total amount spent on participant compensation?
    \answerNA{}
\end{enumerate}

\end{enumerate}

%% file: Appendix.tex
\section{Analytical Solution for Toy Shape-Texture Regression}

% We develop a new versatile dataset designed for domain generalization, called the Shape-Texture dataset, based on the data generating model in Fig~\ref{fig:sem}.
The Shape-Texture dataset involves an invariant latent feature 
$\theta_\causal \in \mathbb{R}$, 
a spurious latent feature 
$\theta_\spu \in \mathbb{R}$, 
and the label $\theta_\labl \in \mathbb{R}$, % for the label.
each of which represents an angle of orientation, 
and therefore is naturally represented on a unit circle in the complex plane as  $e^{i\theta} \in \mathcal{S}^1$.
In the toy linear regression setting, the observed input is simply the concatenated features %{\red scalar complex numbers}
% $\inp = [e^{i\theta_\causal}, e^{i\theta_\spu}]$,
% 
    $x = \begin{bmatrix}
            e^{i\theta_\causal}\\
            e^{i\theta_\spu}
        \end{bmatrix} $
and the label is $\labl = e^{i\theta_\labl}$.

For all task settings, the latent angles and the label angle are 
% the label angle and the invariant feature are 
linearly related as
% via additive Gaussian noise:
$
    \theta_\causal \leftarrow \theta_\labl + \xi_\causal, 
    \theta_\spu \leftarrow \theta_\labl +  \xi_\spu,
$
where $\xi_* \sim \mathcal{N}(0,\sigma_*)$ with probability $p_*$, 
or $\xi_* \sim \mathcal{U}(-\pi,\pi)$ probability $1-p_*$
for $* \in \{\causal,\spu\}$. 
% is additive noise, 
% $p_\causal$, $\sigma_\causal$ are  fixed invariant parameters
% whereas 
% $p_\spu$, $\sigma_\spu$ can vary from one environment to another. In this work, for all task settings, we use two training environments and one testing environment with 
We used 
$p_\causal = 0.75$ for all environments, 
$p_{\spu_1} = 0.8, p_{\spu_2}= 1$ for the two training environments, $p_{\spu_\text{test}}=0$ for the testing environment, 
and $\sigma_\causal=\sigma_{\spu_e}=0$ for all environments.

Note that
\begin{align*}   
    \expect_e[e^{i(\theta_\causal-\theta_\labl)}] 
        &= \begin{cases}
                1,& \quad \text{with prob.} \ \ p_\causal\\
                0,& \quad \text{with prob.} \ \ 1-p_\causal
            \end{cases}\\
    \expect_e[e^{i(\theta_\spu-\theta_\labl)}] 
            &= \begin{cases}
                1,& \quad \text{with prob.} \ \ p_\spu\\
                0,& \quad \text{with prob.} \ \ 1-p_\spu
            \end{cases}\\
    \expect_e[e^{i(\theta_\causal-\theta_\spu)}] 
            &= \begin{cases}
                1,& \quad \text{with prob.} \ \ p_\causal p_\spu\\
                0,& \quad \text{with prob.} \ \ 1-p_\causal p_\spu
            \end{cases}
\end{align*}
which yields
\begin{align}
    \expect_e[xx^\dagger] 
            &= \begin{bmatrix}
                    1 &  \expect_e[e^{i(\theta_\causal-\theta_\spu)}]\\
                    \expect_e[e^{-i(\theta_\causal-\theta_\spu)}] & 1
                \end{bmatrix} 
            = \begin{bmatrix}
                1 &  p_\causal p_\spu\\
                p_\causal p_\spu & 1
            \end{bmatrix} 
    \\
    \expect_e[x \labl^\dagger] 
            &= \begin{bmatrix}
                    \expect_e[e^{i(\theta_\causal-\theta_\labl)}] \\
                    \expect_e[e^{i(\theta_\spu-\theta_\labl)}]
                \end{bmatrix}
            = \begin{bmatrix}
                p_\causal \\
                p_\spu
            \end{bmatrix}
\end{align}

We consider real weights 
% for linear regression
$ w = \begin{bmatrix}
        w_{\causal} &  w_\spu
      \end{bmatrix}  $. 
% For image-based regression we used complex $w$),
Since $\out = w x$, 
\begin{align*}
    \expect_e[\out \out^\dagger] 
    & = w \expect_e[xx^\dagger] w\trans 
    = w_\causal^2 + w_\spu^2 + 2w_\causal p_\causal w_\spu p_{\spu_e}, \\
    \expect_e[\out \labl^\dagger] 
    & = w \expect_e[x \labl^\dagger] 
    = w_\causal p_\causal + w_\spu p_{\spu_e}
\end{align*}

% ERM, 
The risk of environment $e$ is 
$$    \loss_{e}(f)     
= \expect_e[ \Vert o - y \Vert ^2]\\
= \expect_e[oo^\dagger] - (\expect_e[oy^\dagger] + \expect_e[oy^\dagger]^\dagger) + \expect_e[yy^\dagger].
$$ 
Therefore, the average risk for training is 
\begin{align}
    \bar{\loss}_{tr}(f)     
    % &= \frac{1}{2}  \sum_{e \in \{1,2\}}   \loss_{e}(f)    
    = \frac{1}{2}  \sum_{e \in \{1,2\}} (w_\causal^2 + w_\spu^2 + 2w_\causal p_\causal w_\spu p_{\spu_e} - 2(w_\causal p_\causal + w_\spu p_{\spu_e}) + 1)
\end{align}

\paragraph*{MRI constraint}
\begin{align}
    \expect_{e_1}[\out \labl] &= \expect_{e_2}[\out \labl]  
    \nonumber \\
    \implies &  w_\causal p_\causal + w_\spu p_{\spu_1} = w_\causal p_\causal + w_\spu p_{\spu_2} 
    \nonumber \\
    \implies &  w_\spu = 0
\end{align}
Minimizing $\bar{\loss}_{tr}(f)$ subject to the above constraint yields
$(w_\causal^*, w_\spu^*) =  ( p_\causal, 0 )$
as the constrained optimum.

\paragraph*{IRM-relaxed constraint}
\begin{align}
    \expect_{e_1}[\out \out] - \expect_{e_1}[\out \labl] &= \expect_{e_2}[\out \out] - \expect_{e_2}[\out \labl] 
    \nonumber \\
    \implies &  (2w_\causal - p_\causal)w_\spu p_{\spu_1} = (2w_\causal - p_\causal)w_\spu p_{\spu_2} 
    \nonumber \\
    \implies   
    & \begin{cases}
        w_\spu = 0%  \quad \text{or} 
        \\
        w_\causal = \frac{1}{2p_\causal}
    \end{cases}
\end{align}
% 
% The constrained optima for 
Minimizing $\bar{\loss}_{tr}(f)$ subject to the above constraint yields 
$$ (w_\causal^*, w_\spu^*) =  
\begin{cases}
    & (p_\causal, 0) 
    \\
    & (\frac{1}{2p_\causal}, \frac{p_{\spu_1} + p_{\spu_2}}{4})
\end{cases}    
$$
as the constrained optima.

\paragraph*{IRM constraint}
\begin{align}
    \forall e \in \{1,2\} ~~~~~~~
    \expect_{e}[\out \out] - \expect_{e}[\out \labl]
        % &= \expect_{e_2}[\out \out] - \expect_{e_2}[\out \labl] = 0 
        % \\
    % \implies 
    & = w_\causal^2 + w_\spu^2 + 2w_\causal p_\causal w_\spu p_{\spu_e} - w_\causal p_\causal - w_\spu p_{\spu_e} 
    = 0
\end{align}
which yields 
\begin{align}
    % = w_\causal^2 + w_\spu^2 + 2w_\causal p_\causal w_\spu p_{\spu_2} - w_\causal p_\causal - w_\spu p_{\spu_2} = 0 
    %     \nonumber     \\
    % \implies
    & (w_\causal^*, w_\spu^*) =  
    \begin{cases}
        & (0,0) %  \quad \text{or} 
        \\
        &   (p_\causal,  0)%  \quad \text{or} 
        \\
        &   \left(\frac{1}{2p_\causal} ,  \frac{\sqrt{2p_\causal^2 - 1}}{2p_\causal} \right)  %  \quad \text{or} 
        \\
        &   \left(\frac{1}{2p_\causal} , - \frac{\sqrt{2p_\causal^2 - 1}}{2p_\causal} \right)  
    \end{cases}
\end{align}
{\red 
as constraints,
% Minimizing $\bar{\loss}_{tr}(f)$ subject to the constraints above yields
which are also the constrained optima of the problem. % for $\bar{\loss}_{tr}(f)$.
}

\captionsetup[subfigure]{position=top,textfont=normalfont,singlelinecheck=off,justification=raggedright}
\begin{figure}[t]
	\begin{center}
	\subfloat[Train environment $e_1$: $p_\causal = 0.75, \ p_\spu = 1$]{\includegraphics[width=0.8\columnwidth]{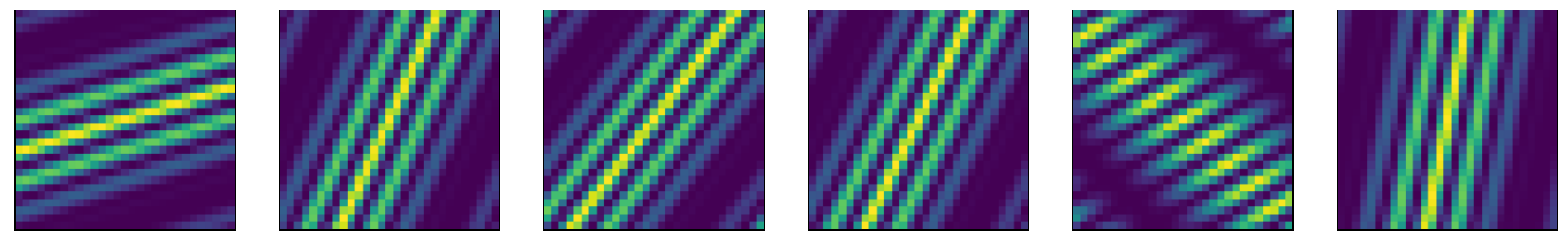}}\quad
	\subfloat[Train environment $e_2$: $p_\causal = 0.75, \ p_\spu = 0.8$]{\includegraphics[width=0.8\columnwidth]{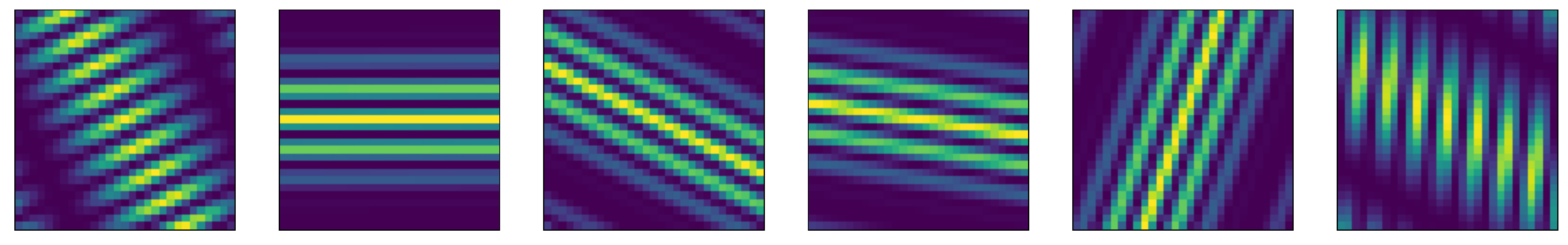}}\quad
	\subfloat[Test environment $e_0$: $p_\causal = 0.75, \ p_\spu = 0$]{\includegraphics[width=0.8\columnwidth]{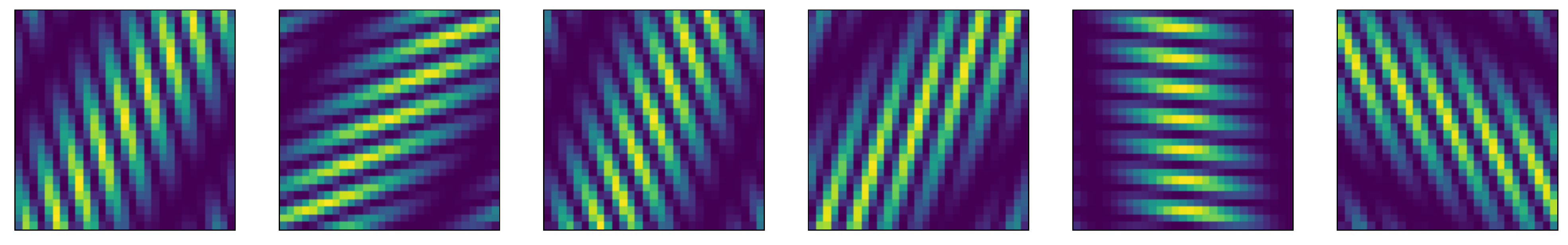}}\quad
	\caption{
	Sample Shape-Texture dataset images.
	}
	\label{fig:images_shape_texture}
	\end{center}
\end{figure}

\captionsetup[subfigure]{position=top,textfont=normalfont,singlelinecheck=off,justification=raggedright}
\begin{figure}[t]
	\begin{center}
	\subfloat[Shape-Texture Classification]{\includegraphics[width=\columnwidth]{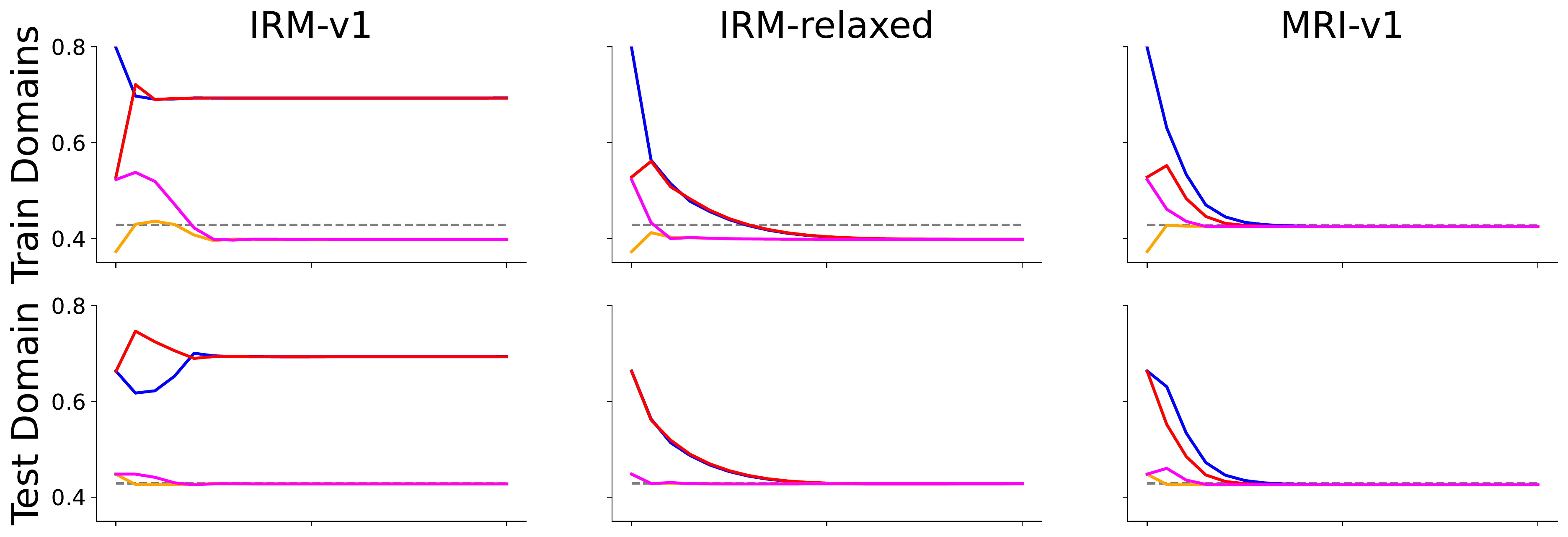}}\quad
	\subfloat[Toy-CMNISTa]{\includegraphics[width=\columnwidth]{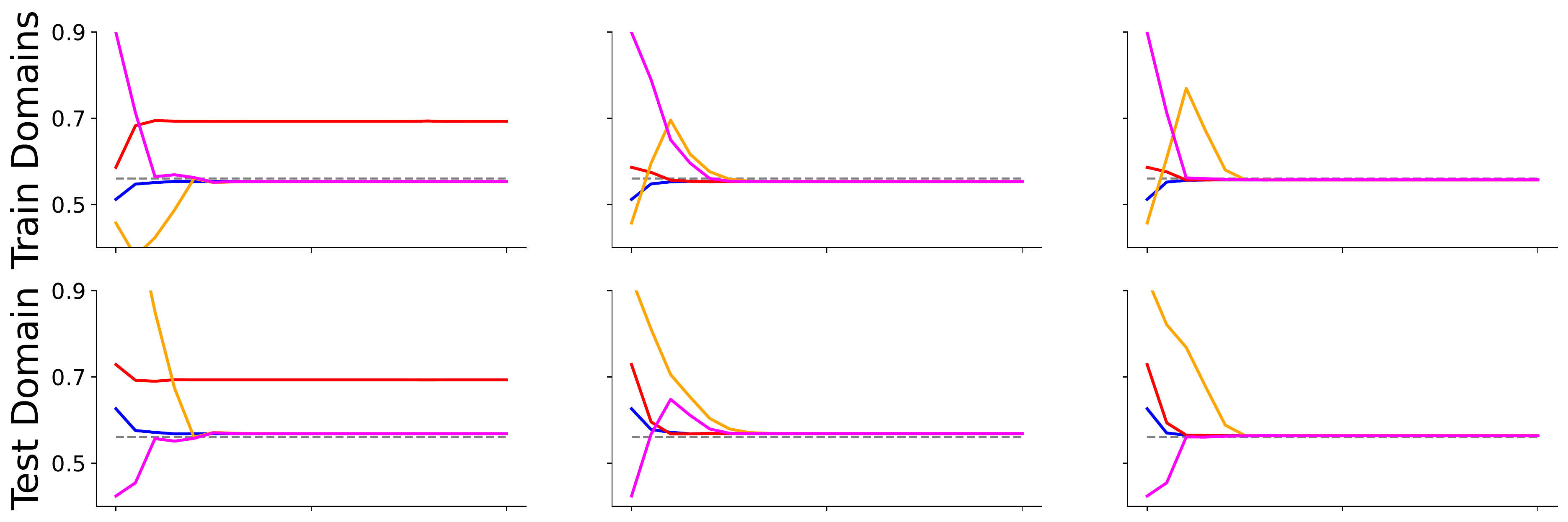}}\quad
	\subfloat[Toy-CMNISTb]{\includegraphics[width=\columnwidth]{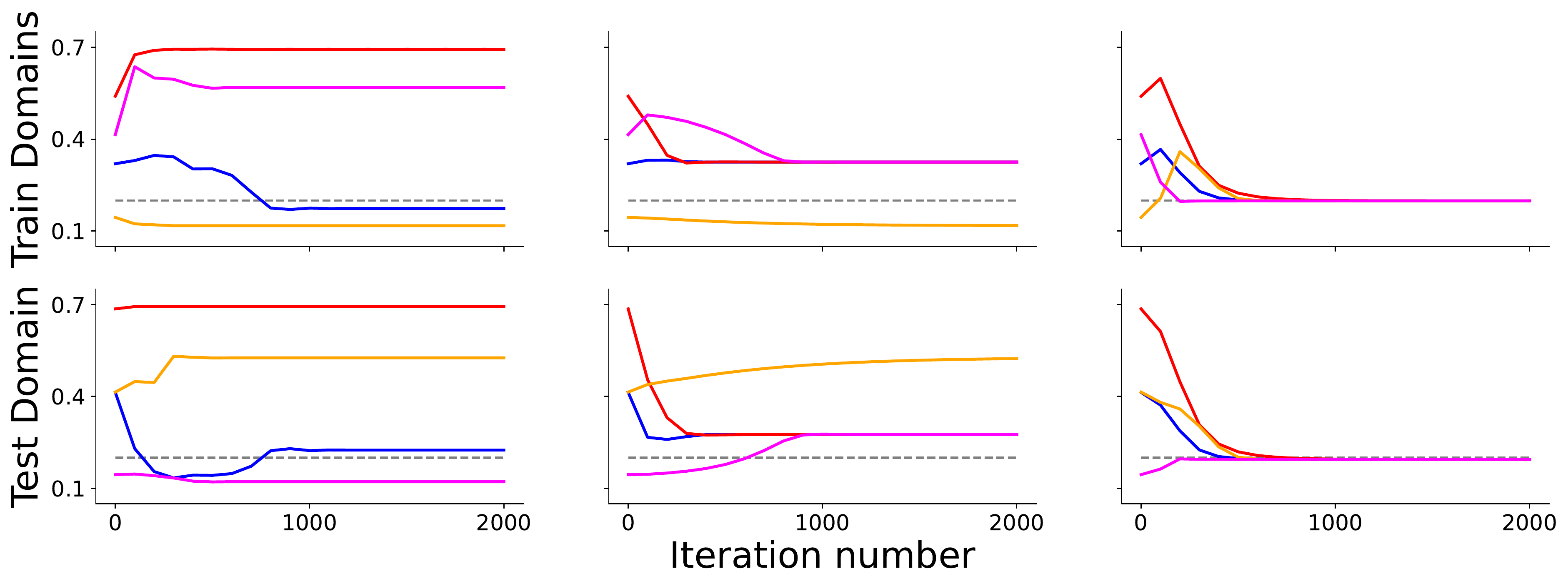}}\quad
	\caption{
	Linear classification tasks. Average train domain risk $\loss_\text{tr}$ (top) and test domain risk $\loss_\text{test}$ (bottom) for four different weight initializations on (A) toy Shape-Texture classification dataset, (B) toy-CMNISTa dataset, and (c) toy-CMNISTb dataset (C). Each color corresponds to a different weight initialization. Black dotted line shows the optimal invariant predictor performance.}
	\label{fig:linear_classification}
	\end{center}
\end{figure}

\section{Generalized version of IRM-v1}

In section \ref{sec:IRMv1}, we presented the linear version of IRM, IRM-v1, which is obtained by limiting the output perturbation in eq~\eqref{eq:IRM_loss_perturb} to the space of linear functions. In this section, we extend our analysis to a more general case, where the output perturbation is restricted to the space of functions defined by
$\delta\phi(\out; \vec{\xi}) =  \sum_{ i \in \mathcal{I}} \delta\phi_i(\out)  \xi_i $,
where $\{\delta\phi_i(\out) \}_{ i \in \mathcal{I}}$ denotes the basis set for the function space. 
This reduces eq~\eqref{eq:IRM_loss_perturb} to

\begin{align}
	\label{eq:IRM_loss_perturb_general}
	\forall e \in \mathcal{E}, 
	~~~~~	\delta  \loss_e(f)  
    % &= \expect_{P_e(\out)} [  \expect_{P_e( \labl | \out)}  [ \, \partial_\out l(\out,\labl) \cdot \out \, ]]  
	&= \sum_{\red i \in \mathcal{I}}  \expect_{e} [ \, \partial_\out l(\out,\labl) \cdot \delta\phi_i(\out) \,]   \xi_i
	% \nonumber \\
    % & = \expect_e [  \out\sigma(\out) ]  - \expect_e [  \out \labl  ] 
	= 0 .
\end{align}
for arbitrary $\vec{\xi}$.

Since $ \partial_\out l(\out,\labl) \cdot \phi(\out)  = 
 \phi(\out) ( \sigma(\out) - \labl) $
for standard loss functions$^\text{\ref{sigma}}$,
eq~\eqref{eq:IRM_loss_perturb_linear} %IRM-v1 
% For each environment $e$,
is equivalent to % {\green requires/imposes} 
\begin{align}
	\label{eq:IRM_general}
	\forall e \in \mathcal{E}, {\red \forall i \in \mathcal{I}}, ~~~~~~
	\expect_e [  \delta\phi_i(\out) \labl  ] = \expect_e [  \delta\phi_i(\out)\sigma(\out) ] .
 \end{align}
 Note that eq~\eqref{eq:IRM_general} does not prescribe any meaningful invariance relationship, since the RHS term $\expect_e [  \phi(\out)\sigma(\out) ]$ depends on the environment.

\section{Training Procedure}

PyTorch version 1.10.0 was used. For all tasks, Tesla V-100 with a memory of 16GB was used. For all our simulations, we utilized the DomainBed suite (\cite{gulrajani2020search}; available under MIT license).

\subsection{Linear Datasets} For Fig~\ref{fig:loss_landscape}, \ref{fig:linear_classification}, and \ref{fig:toy_CMNIST}, we used a linear predictor $\predictor(\inp) =  w_\causal \latent_\causal+ w_\spu \latent_\spu$. In all cases, we used full-batch training with a batch size of 40000.
We optimized using the penalty method with $\mu = 5\times10^4$. For Fig~\ref{fig:loss_landscape}, we used the SGD optimizer with no momentum. For Fig~\ref{fig:linear_classification} and \ref{fig:toy_CMNIST}, we used the Adam optimizer with betas of (0.975, 0.999). The learning rate was set to 0.005, the total number of gradient steps were set to 2000, and the gradient norms were clipped at 2 for all cases.

% For Fig~\ref{fig:finite_samples}, we used a one-hidden-layer ReLU network where the hidden layer contained 50 neurons. For all cases except Fig~\ref{fig:finite_samples}c, we used full-batch training. We optimized using the ALM method with $\mu = 1$ and $\lambda = 5$ for IRM and $\lambda = 10$ for MRI. We used the Adam optimizer with betas of (0.975, 0.999). The learning rate was set to 0.02. The total number of gradient steps were set to 4000 for IRM and 2000 for MRI. No gradient clipping was used. For all cases, we used 10 random seeds.

\subsection{Non-Linear Datasets}  The results for the nonlinear datasets are presented in Table~\ref{tab:main_results}-\ref{tab:vlcs_acc_results} and Fig~\ref{fig:ST_PM}, \ref{fig:CMNIST_PM}, and, \ref{fig:ST_other}. The results in Table~\ref{tab:main_results}-\ref{tab:vlcs_acc_results} were produced from 10 different model initialization using random seeds.
%
% For all cases, we use 10 random seeds unless mentioned otherwise.

\paragraph{Shape-Texture Datasets} For the shape-texture datasets, we used a single convolutional layer with 12 channels, a kernel size of 5, and stride of 1. The convolutional layer was followed by a 2D adaptive average pooling layer with output size of 1x1 followed by a fully-connected linear layer. For results in Table~\ref{tab:main_results}, for optimization using the ALM method we used $\mu = 10$ and $\lambda = 10$ and for optimization using PM we used $\mu = 10000$. We used mini-batch training with a mini-batch size 8000 and total dataset size 50000. We used the Adam optimizer with betas of (0.9, 0.999). The learning rate was set to 0.01. The total number of gradient steps were set to 2000. Gradient clipping was only used when training using PM such that gradient norms were clipped at 2. We used a weight decay of 0.003 for the classification setting and 0.005 for the regression setting.

For Fig~\ref{fig:ST_PM} and Fig~\ref{fig:ST_other}, we used the same training procedure as above for the penalty method except that the seeds and algorithm-specific parameters were chosen at random to generate 50 random hyperparameter combinations for each algorithm. For IRM-v1 and MRI-v1, $\mu$ was chosen randomly from the distribution $10^{Uniform[2, 5]}$. We used the DomainBed \citep{gulrajani2020search} implementation of IB-IRM, GroupDRO, and MMD. For IB-IRM, $\lambda_{IRM}$ and $\lambda_{IB}$ were both chosen randomly from the distribution $10^{Uniform[2, 5]}$. For GroupDRO, $\eta_{GroupDRO}$ was chosen from the distribution $10^{Uniform[-3, 1]}$. For MMD, $\gamma_{MMD}$ was chosen from the distribution $10^{Uniform[-1, 1]}$.

\paragraph{CMNIST} For the CMNIST datasets, we used the LeNet architecture \citep{lecun1998gradient}.  We used mini-batch training with a mini-batch size of 256. For optimization using the ALM method with $\lambda = 10$ for both datasets, but $\mu = 1$ for CMNISTa and $\mu = 10$ for CMNISTb datasets for results in Table~\ref{tab:main_results}. For optimization using the PM method, we used $\mu = 10000$ for both datasets. We used the Adam optimizer with betas of (0.975, 0.999). The learning rate was set to 0.001. The total number of gradient steps were set to 10000. The weight decay was set to 0.001. Gradient clipping was only used when training using PM such that gradient norms were clipped at 2. For the CMNIST datasets, we also used data augmentation during training. This was done using Torchvision's Random affine transformation class with degrees=30, translate=(0.1, 0.1), and scale=(0.8, 1.2). 

For Fig~\ref{fig:CMNIST_PM}, we used the same training procedure as above for training except that the seeds and algorithm-specific parameters were chosen at random to generate 50 random hyperparameter combinations for each algorithm. For training with PM, $\mu$ was chosen independently and randomly from the distribution $10^{Uniform[2, 5]}$. For training with ALM, $\lambda$ and $\mu$ were chosen randomly from the distribution $Uniform[1, 10]$. The total number of gradient steps were set to 5000.

\paragraph{VLCS/Terra-Incognita} 
{ \green VLCS \citep{fang2013unbiased} and Terra-Incognita \citep{beery2018recognition} are popular domain generalization datasets that contain realistic images. For the binary classification tasks, we only used two classes from each dataset (VLCS: bird and person, Terra-Incognita: opossum and raccoon). We used the LeNet architecture \citep{lecun1998gradient} for both the datasets. We used
mini-batch training with a mini-batch size of 256. We trained the models using the ALM method with $\mu = 10$ and $\lambda = 10$. We used the Adam optimizer with betas of (0.975, 0.999). The learning rate was set to 0.001. The total number of gradient steps were set to 2000. The weight decay was set to 0.001. No gradient clipping was used. The images were resized to the shape (64,64,3) and normalized with the mean (0.485, 0.456, 0.406) and standard deviation (0.229, 0.224, 0.225). For data augmentation, we used Torchvision's RandomHorizontalFlip transform, RandomGrayscale transform, RandomResizedCrop transform with scale = (0.7, 1.0), and ColorJitter transform with brightness, contrast, saturation, and hue set to 0.3.}

% For both versions of the shape-texture tasks, we used the same training procedure. We used a convolutional neural network with a single convolutional layer with 12 channels, a kernel size of 5, and stride of 1. The convolutional layer was followed by a 2D adaptive average pooling layer with output size of 1x1. The pooling layer is followed by a fully-connected linear layer. For optimization, we used the Adam optimizer with betas of (0.975, 0.999) and a learning rate 0.01. We optimized the mean squared error loss function. We clipped the gradient norms at 0.5. The batch size was set to 8000. The total number of steps were set to 2000. For this task, we used 5 different seeds. For this task, we train all the algorithms for 5 different seeds for 5 different penalty coefficients $\lambda$ 300, 1000, 3000, 10000, and 30000.

% \textsc{\bf Colored-MNIST task}  For this task, we used LeNet-5 introduced by \cite{lecun1998gradient}. For optimization, we used the Adam optimizer with betas of (0.96, 0.999) and a learning rate 0.0005. Logistic loss function was used for the binary classification. We clipped the gradient norms at 0.2. The batch size was set to 8000. The total number of steps were set to 4000. For this task, we used 4 different penalty coefficients $\lambda$ 300, 3000, 10000, and 30000 for 5 different seeds.

\begin{table}
% \captionabove{Risk}
\begin{subtable}[c]{1\textwidth}
\centering
\vspace*{1mm}
\caption*{Risk}
\scalebox{0.7}{
    \begin{tabular}{M{2.5cm}|cc|cc} \toprule
        &  \multicolumn{2}{c|}{VLCS} &
     \multicolumn{2}{c}{Terra-Incognita} \\
      & Train & Test & Train & Test \\ \midrule
    ERM
    & $0.12 \pm 0.0 $ & $0.92 \pm 0.02 $ 
    & $0.41 \pm 0.01 $ & $1.52 \pm 0.12 $  
    \\[0.3cm]
    IRM-v1 (ALM)
    & $0.14 \pm 0.01 $ & $1.03 \pm 0.04 $ 
    & $0.45 \pm 0.01 $ & $\bm{1.30} \pm 0.21 $  
    \\[0.3cm] 
    MRI-v1 (ALM)
    & $0.24 \pm 0.01 $ & $\bm{0.64} \pm 0.01 $ 
    & $0.46 \pm 0.01 $ & $\bm{1.30} \pm 0.07 $  
    \\[0.05cm]
    \bottomrule
    \end{tabular}}
% \subcaption{Risk}

\end{subtable}

\begin{subtable}[c]{1\textwidth}
\centering
\vspace*{1mm}
\caption*{Accuracy}
\scalebox{0.7}{
    \begin{tabular}{M{2.5cm}|cc|cc} \toprule
        &  \multicolumn{2}{c|}{VLCS} &
     \multicolumn{2}{c}{Terra-Incognita} \\
      & Train & Test & Train & Test \\ \midrule
    ERM
    & $0.96 \pm 0.0 $ & $\bm{0.75} \pm 0.01 $ 
    & $0.81 \pm 0.01 $ & $\bm{0.39} \pm 0.01 $  
    \\[0.3cm] 
    IRM-v1 (ALM)
    & $0.95 \pm 0.0 $ & $\bm{0.74} \pm 0.01 $ 
    & $0.78 \pm 0.01 $ & $\bm{0.38} \pm 0.03 $  
    \\[0.3cm] 
    MRI-v1 (ALM)
    & $0.90 \pm 0.01 $ & $0.71 \pm 0.01 $ 
    & $0.78 \pm 0.01 $ & $0.36 \pm 0.01 $  
    \\[0.05cm]
    \bottomrule
    \end{tabular}}
% \subcaption{Accuracy}
\end{subtable}
\caption{Performance of different algorithms on VLCS and Terra-Incognita datasets. (Top) Averaged train and test domain risk and (bottom) accuracy. Mean and standard deviation shown up to 2 decimal places.}
\label{tab:vlcs_acc_results}
\end{table}
% \begin{figure*}[ht]
% 	\vskip 0.2in
% 	\begin{center}
% 	\centering
% 	\includegraphics[width=0.8\columnwidth]{figures/toylinearclassification_panel1.pdf}
% 	\includegraphics[width=0.8\columnwidth]{figures/toylinearclassification_panel2.pdf}
% 	\caption{Linear shape-texture classification. Training loss (top) and testing loss (bottom) corresponding to four different weight initialization. Black dotted line shows the optimal invariant predictor performance.}
% 	\label{fig:linearclassification}
% 	\end{center}
% 	\vskip -0.2in
% \end{figure*}

% \begin{figure*}[t]
% 	\vskip 0.2in
% 	\begin{center}
% 	\centering
% 	\includegraphics[width=0.9\columnwidth]{figures/SPvsALM_panelsupp_1.pdf}
% 	\includegraphics[width=0.9\columnwidth]{figures/SPvsALM_panelsupp_2.pdf}
% 	\caption{Quadratic penalty vs ALM - Classification}
% 	\label{fig:finite_samples}
% 	\end{center}
% 	\vskip -0.2in
% \end{figure*}

% \captionsetup[subfigure]{position=top,textfont=normalfont,singlelinecheck=off,justification=raggedright}
\begin{figure}[t]
	\begin{center}
	% \subfloat[Risk]
    {\includegraphics[width=0.45\columnwidth]{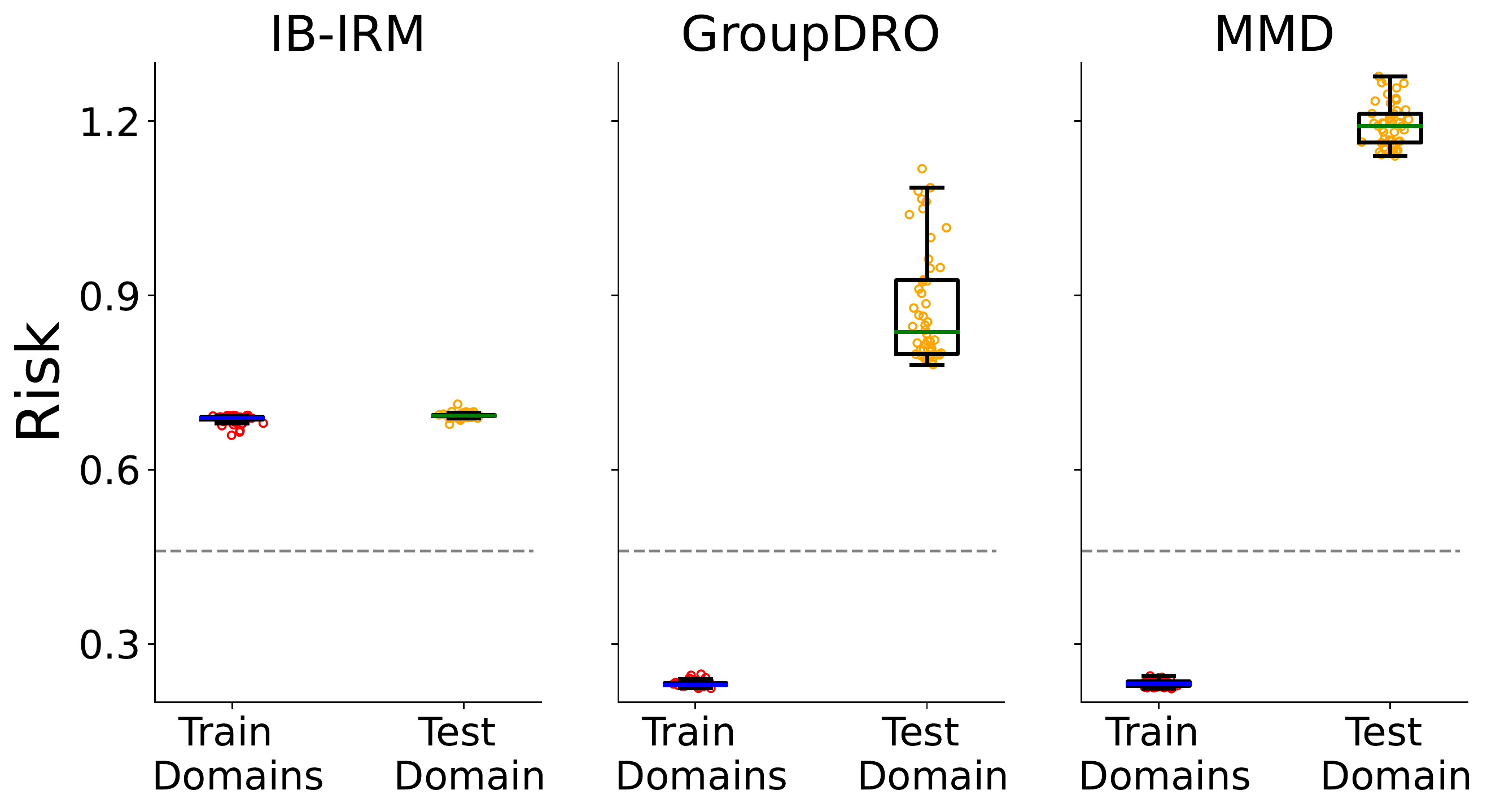}}\quad
	% \subfloat[Accuracy]
    {\includegraphics[width=0.45\columnwidth]{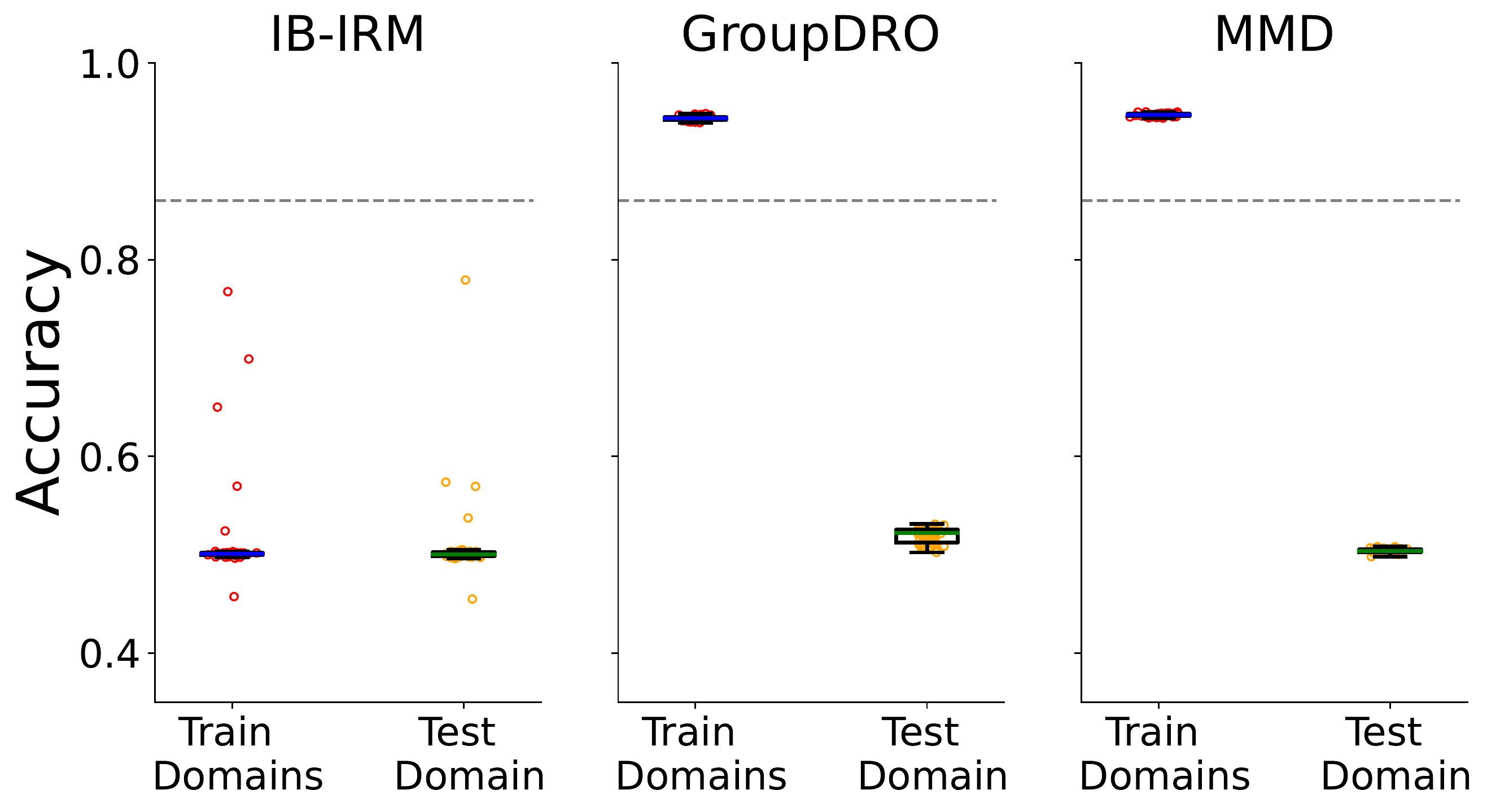}}\quad
	\caption{
	Performance of IB-IRM, GroupDRO and MMD  on shape-texture classification dataset. 
	(Left) averaged train and test domain risk 
	and (right) accuracy 
	for 50 randomly drawn hyperparameters. Grey dashed line denotes the Oracle performance.
}
	\label{fig:ST_other}
	\end{center}
\end{figure}

\captionsetup[subfigure]{position=top,textfont=normalfont,singlelinecheck=off,justification=raggedright}
\begin{figure}[t]
	\begin{center}
	\subfloat[Toy-CMNISTa: Constraint]{\includegraphics[width=0.3\columnwidth]{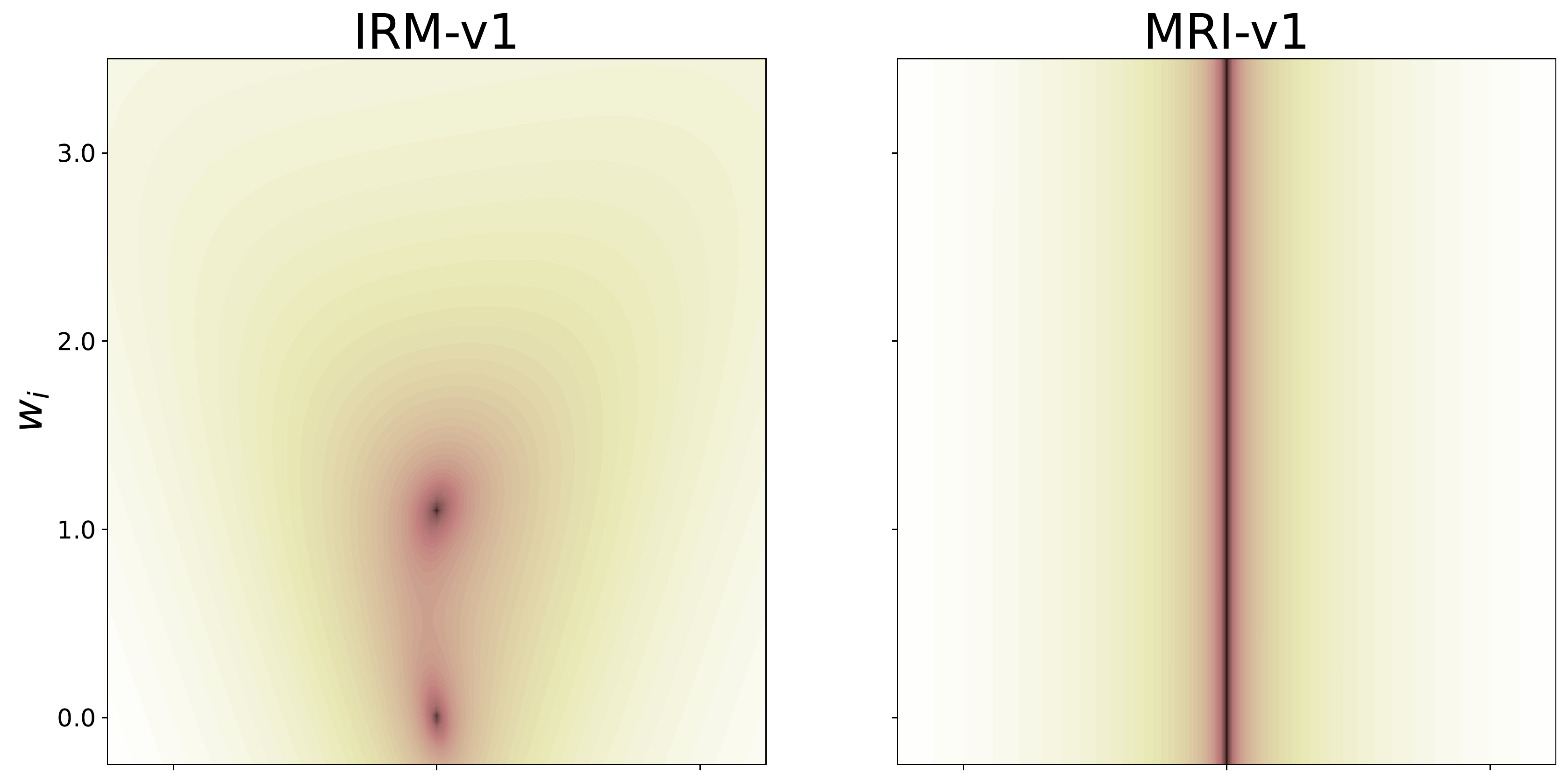}}\quad
	\subfloat[Toy-CMNISTa: Risk]{\includegraphics[width=0.3\columnwidth]{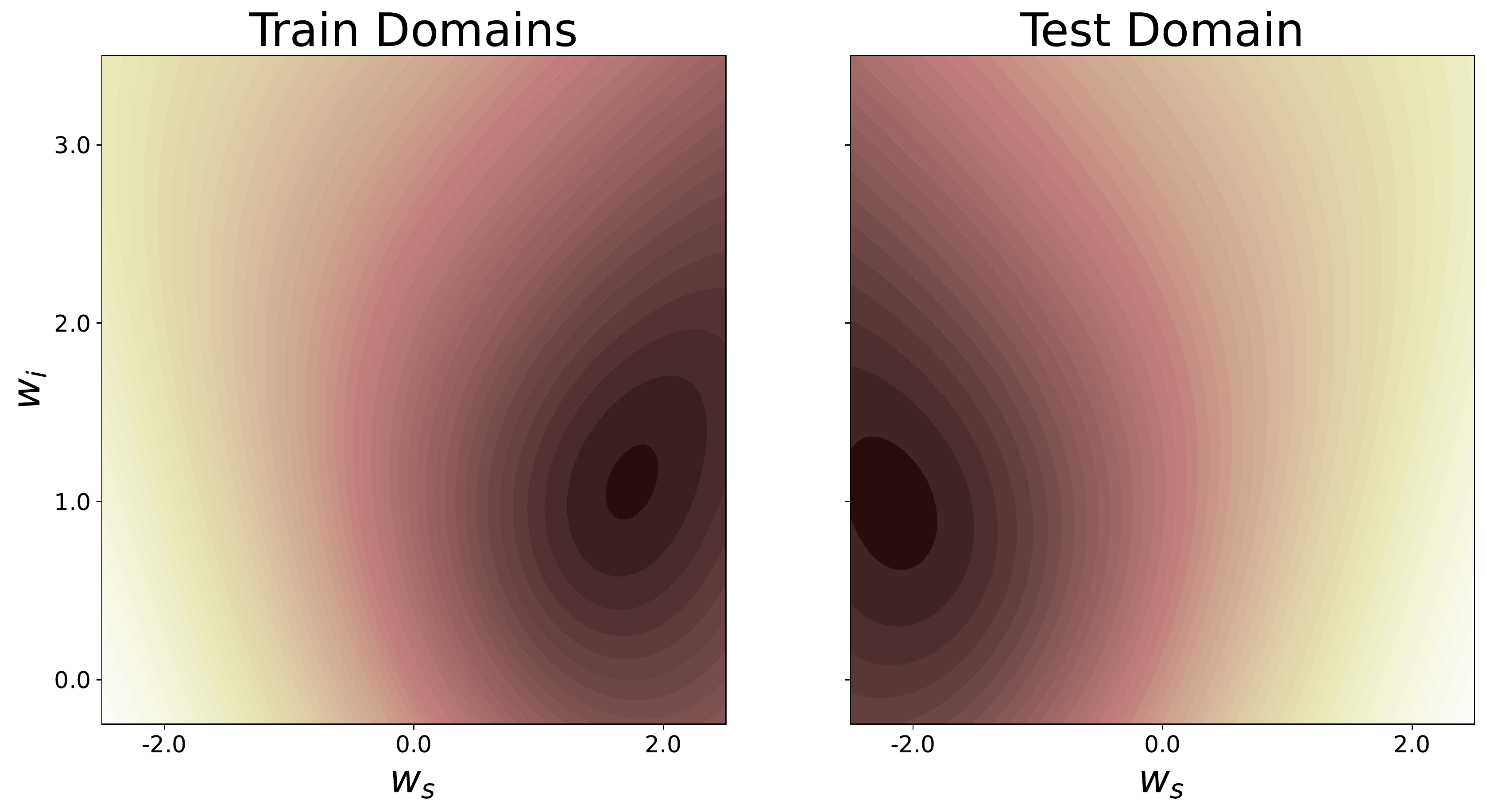}}\quad
	\subfloat[Toy-CMNISTa: Accuracy]{\includegraphics[width=0.3\columnwidth]{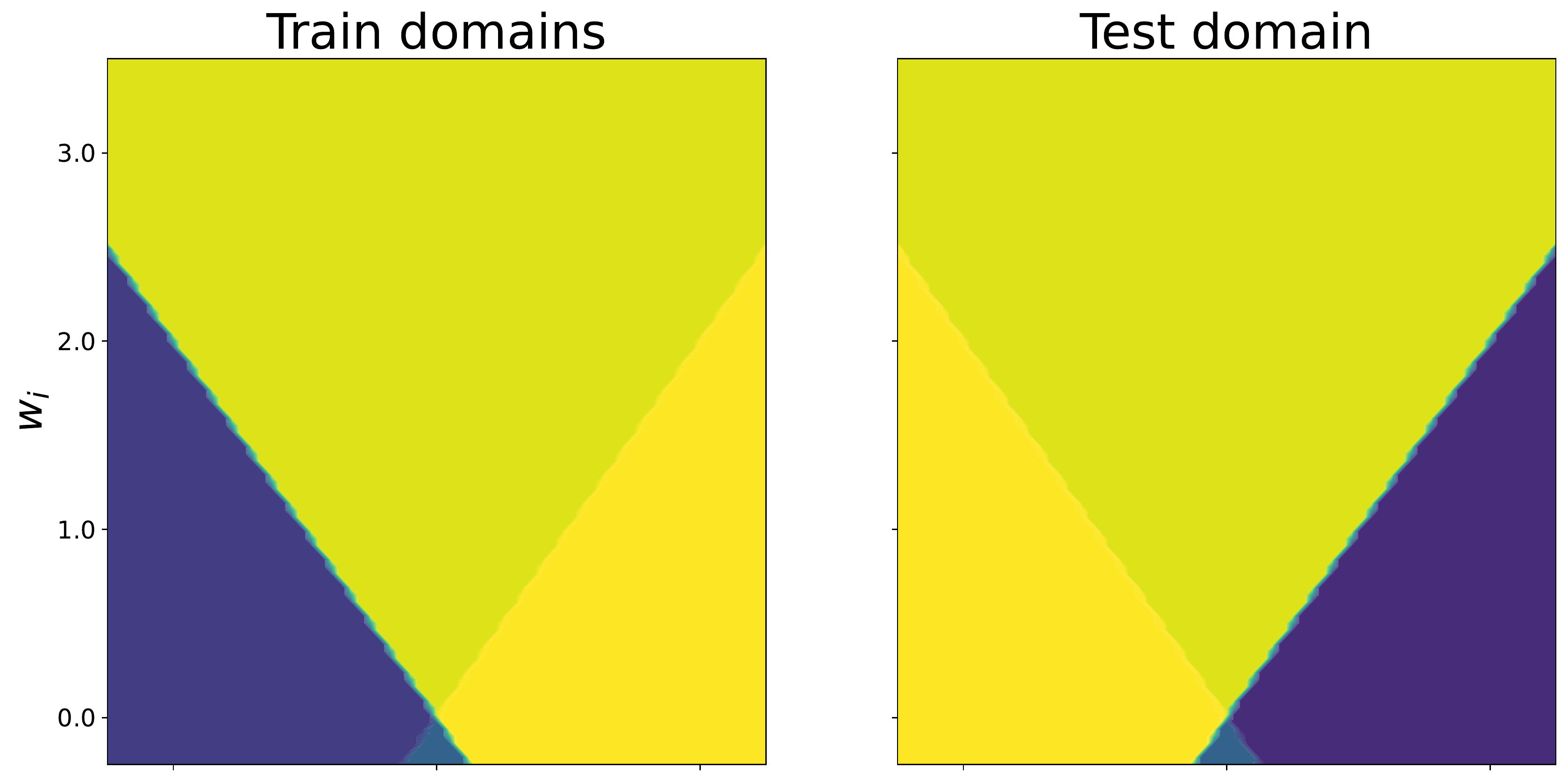}}\quad
	\subfloat[Toy-CMNISTb: Constraint]{\includegraphics[width=0.3\columnwidth]{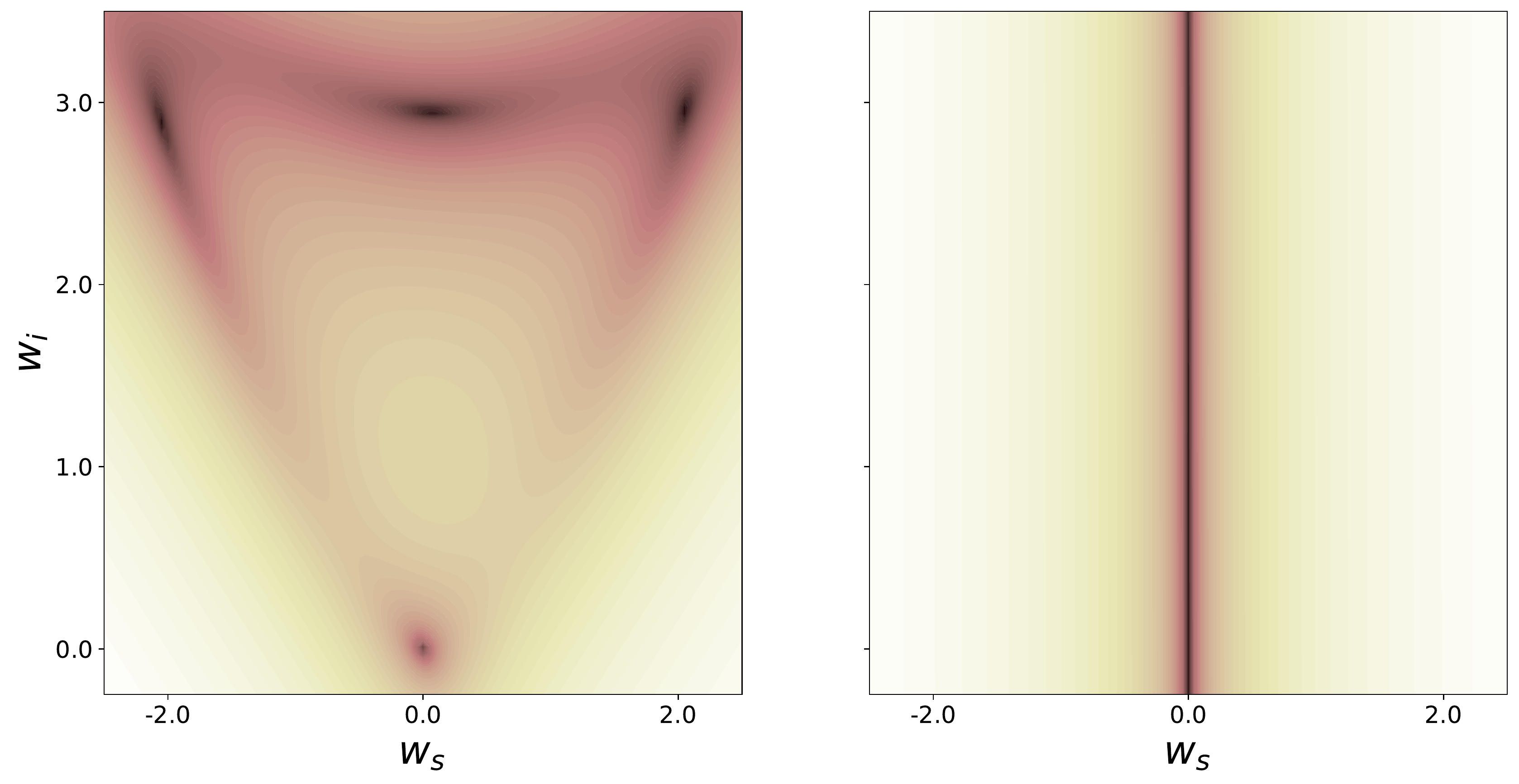}}\quad
	\subfloat[Toy-CMNISTb: Risk]{\includegraphics[width=0.3\columnwidth]{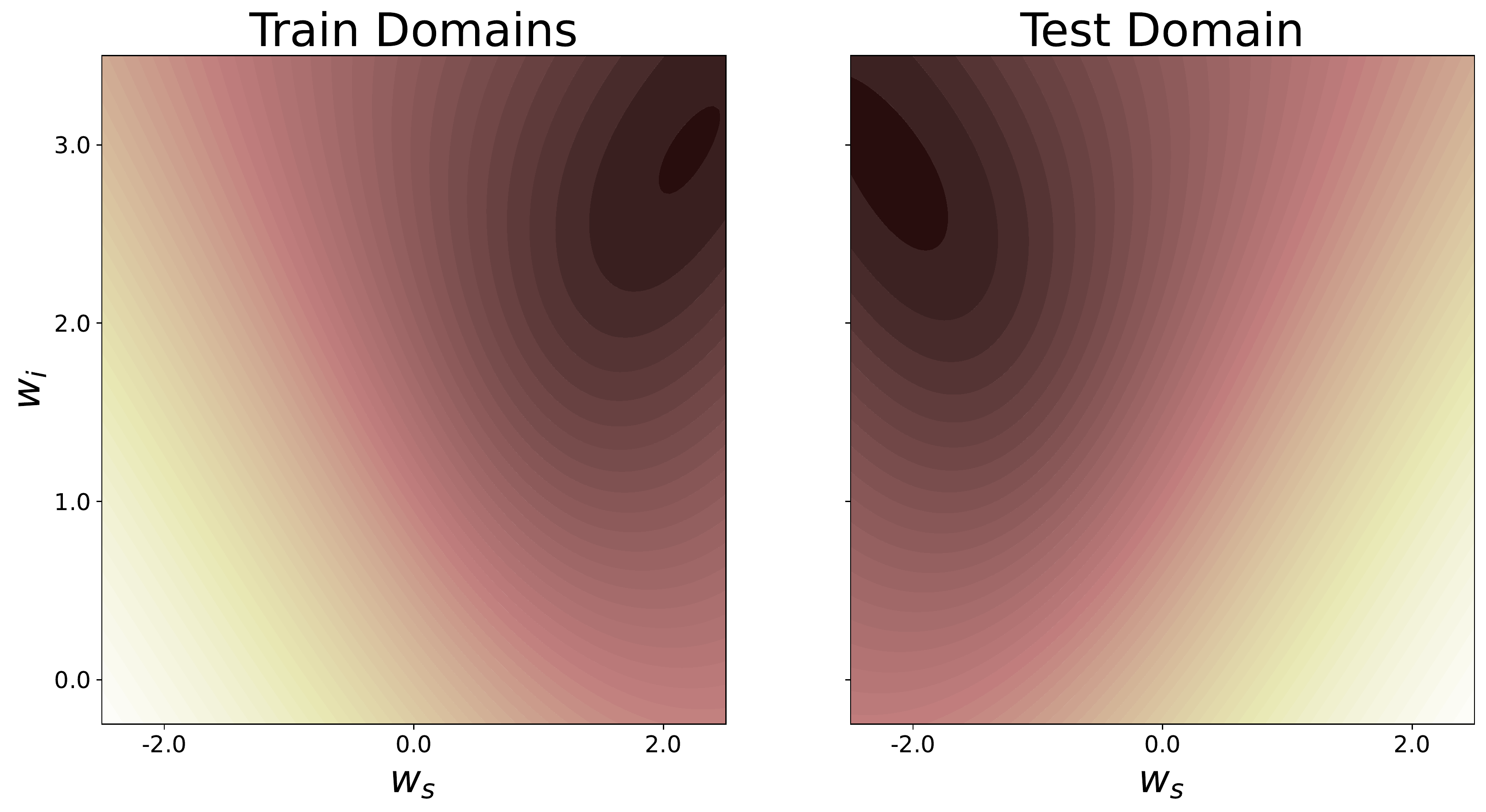}}\quad
	\subfloat[Toy-CMNISTb: Accuracy]{\includegraphics[width=0.3\columnwidth]{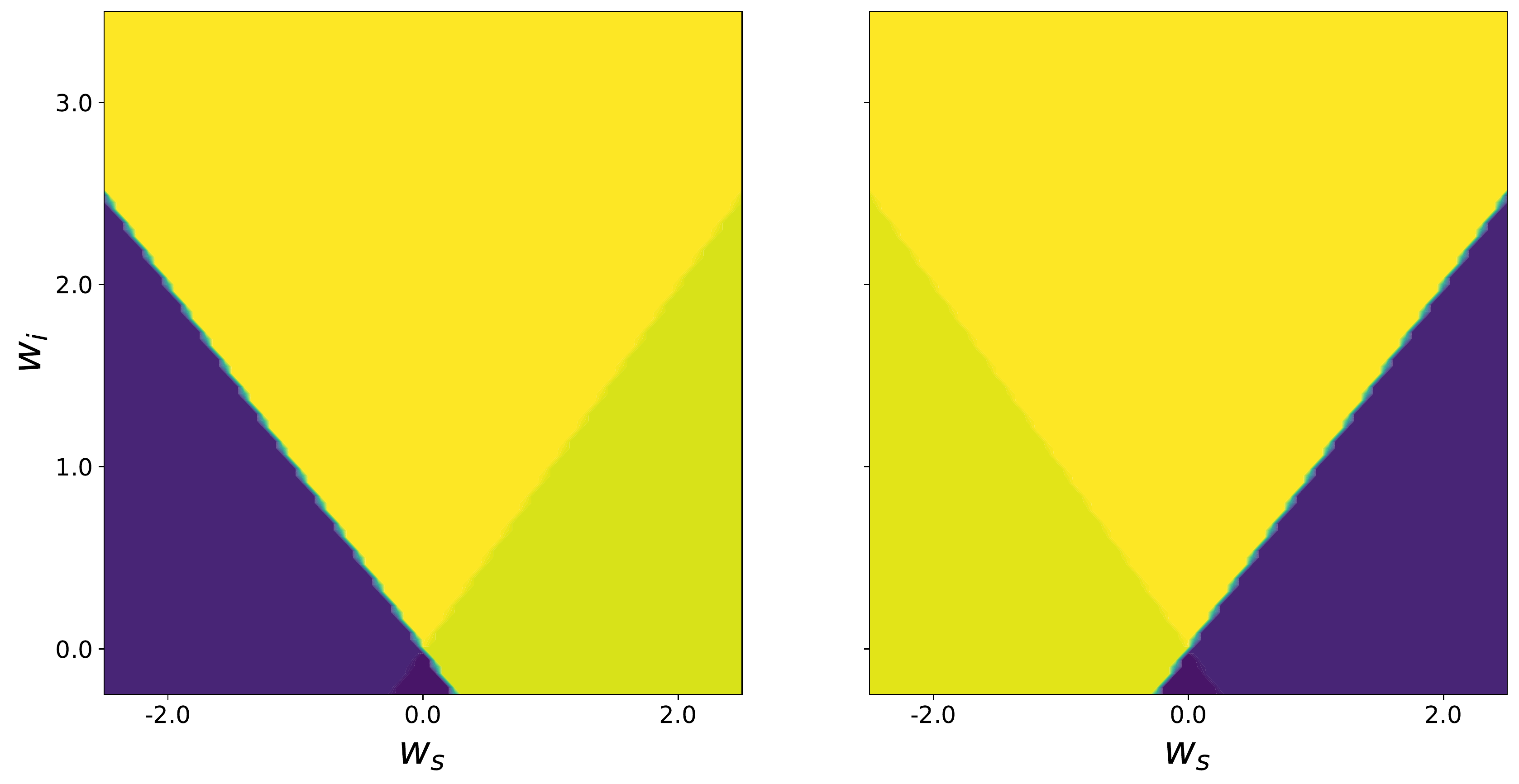}}\quad
	\caption{Linear classification of Toy-CMNISTa/b datasets. 
% 	Results are shown for the linear predictor (Eq~\ref{eq:outputrelation}). 
    (A) Squared constraint function $\left\lVert \vec{c}(\predictor) \right\rVert ^2$ for IRM-v1 (left) and MRI-v1(right) on toy-CMNISTa dataset. Y-axis: weight for the invariant feature $w_\causal$. X-axis: weight for the spurious feature $w_\spu$. (B) Averaged risk over train (left) and test (right) domains as a function of the predictor weights on toy-CMNISTa dataset.
	(C) Averaged accuracy over train (left) and test (right) domains as a function of the predictor weights on toy-CMNISTa dataset.
	(D)-(F) Same as panels A-C except on toy-CMNISTb dataset. 
	}
	\label{fig:toy_CMNIST}
	\end{center}
\end{figure}

\begin{figure}[t]
	\begin{center}
	\centering
	\includegraphics[width=0.9\columnwidth]{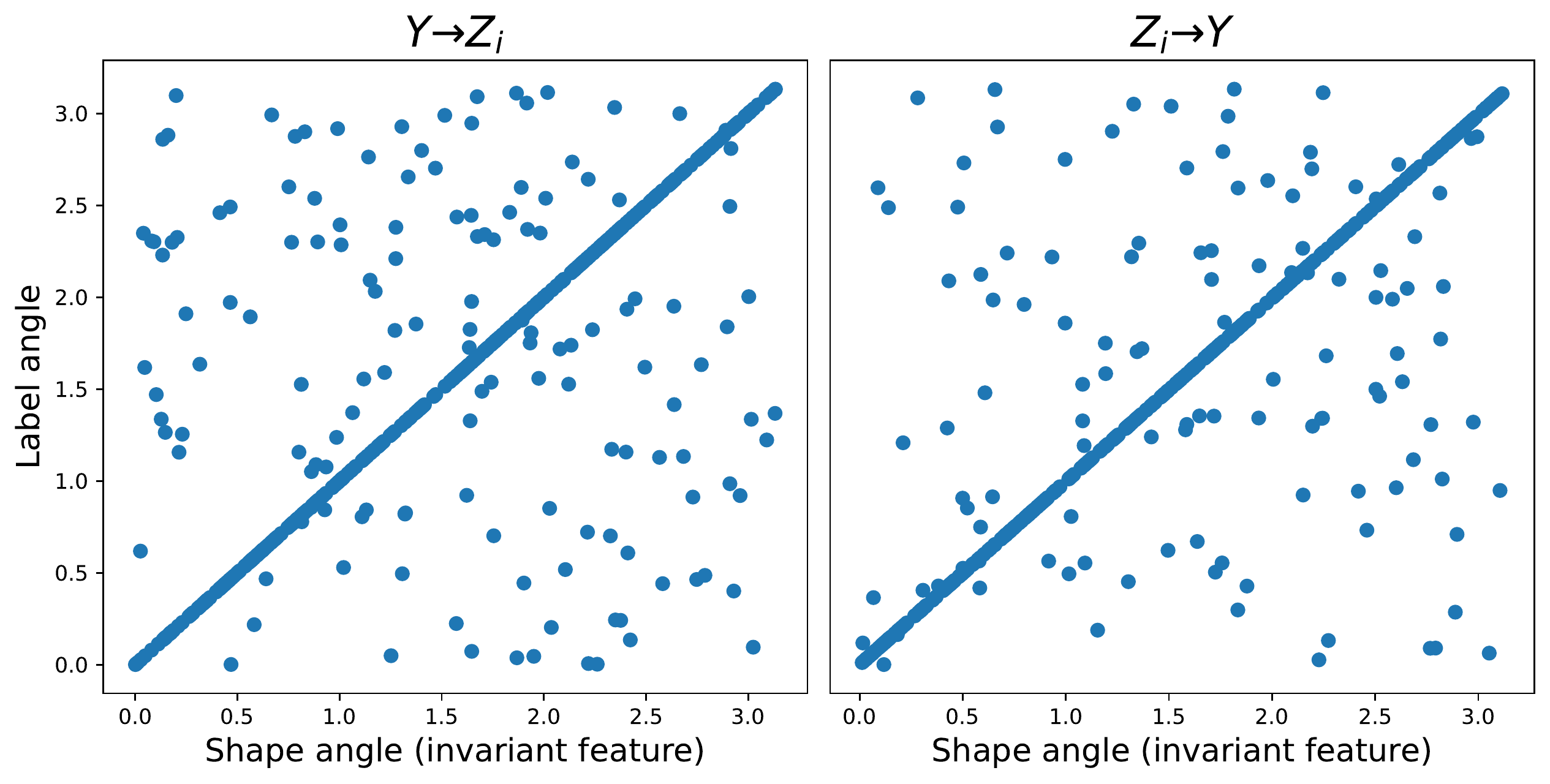}
	\caption{Joint distribution of label angle and shape angle (invariant feature) for the Shape-Texture datasets. This shows that the Shape-Texture dataset generated by either causal directions $\Latent_{\causal} \to \Labl$ or $ \Labl \to \Latent_{\causal}$ are equivalent.} 
	\label{fig:joint_dist}
	\end{center}
\end{figure}